\documentclass[twoside,11pt]{article}

\usepackage{jmlr2e}
\usepackage{amsmath,amssymb,mathrsfs,bbm, comment, color, xcolor}
\usepackage{epsfig,epsf,subcaption,graphicx,graphics, algorithm, algorithmic}

\usepackage{xifthen}

\newcommand{\aaaa}{\mathrm{(a)}}
\newcommand{\bbbb}{\mathrm{(b)}}
\newcommand{\cccc}{\mathrm{(c)}}
\newcommand{\dddd}{\mathrm{(d)}}
\newcommand{\eeee}{\mathrm{(e)}}
\newcommand{\ffff}{\mathrm{(f)}}
\newcommand{\gggg}{\mathrm{(g)}}
\newcommand{\hhhh}{\mathrm{(h)}}

\newcommand{\Prob}[1]{\mathbb{P} \left( #1 \right)}
\newcommand{\lp}{\left(}
\newcommand{\rp}{\right)}
\newcommand{\lb}{\left[}
\newcommand{\rb}{\right]}
\newcommand{\lbp}{\left\{}
\newcommand{\rbp}{\right\}}

\newcommand{\wtild}{\widetilde}
\newcommand{\mb}{\mathbf}

\newcommand{\diag}{\text{diag}}

\DeclareMathOperator*{\argmin}{arg\,min}

\newcommand{\bmat}{\left[\begin{matrix}}
\newcommand{\emat}{\end{matrix}\right]}

\newcommand{\trace}[1]{\textbf{tr}\lp #1\rp}

\allowdisplaybreaks

\ShortHeadings{Redundancy Techniques for Distributed Optimization}{Karakus, Sun, Diggavi, Yin}
\firstpageno{1}

\begin{document}

\title{Redundancy Techniques for Straggler Mitigation in Distributed Optimization and Learning}

\author{\name Can Karakus \email karakus@ucla.edu  \\
       \addr Department of Electrical and Computer Engineering\\
       University of California, Los Angeles\\
       Los Angeles, CA 90095, USA
       \AND
       \name Yifan Sun \email ysun13@cs.ubc.ca \\
       \addr Department of Computer Science\\
	University of British Columbia\\
       Vancouver, BC, Canada
       \AND
       \name Suhas Diggavi \email suhas@ee.ucla.edu \\
       \addr Department of Electrical and Computer Engineering\\
       University of California, Los Angeles\\
       Los Angeles, CA 90095, USA
       \AND
       \name Wotao Yin \email wotaoyin@math.ucla.edu \\
       \addr Department of Mathematics\\
       University of California, Los Angeles\\
       Los Angeles, CA 90095, USA
       }

\maketitle

\begin{abstract}%   <- trailing '%' for backward compatibility of .sty file
Performance of distributed optimization and learning systems is bottlenecked by ``straggler'' nodes and slow communication links, which significantly delay computation. We propose a distributed optimization framework where the dataset is ``encoded'' to have an over-complete representation with built-in redundancy, and the straggling nodes in the system are dynamically left out of the computation at every iteration, whose loss is compensated by the embedded redundancy. We show that oblivious application of several popular optimization algorithms on encoded data, including gradient descent, L-BFGS, proximal gradient under data parallelism, and coordinate descent under model parallelism, converge to either approximate or exact solutions of the original problem when stragglers are treated as erasures. These convergence results are deterministic, \emph{i.e.}, they establish sample path convergence for arbitrary sequences of delay patterns or distributions on the nodes, and are independent of the tail behavior of the delay distribution. We demonstrate that equiangular tight frames have desirable properties as encoding matrices, and propose efficient mechanisms for encoding large-scale data. We implement the proposed technique on Amazon EC2 clusters, and demonstrate its performance over several learning problems, including matrix factorization, LASSO, ridge regression and logistic regression, and compare the proposed method with uncoded, asynchronous, and data replication strategies.
\end{abstract}


\section{Introduction}\label{sec:introduction}
Solving learning and optimization problems at present scale often requires parallel and distributed implementations to deal with otherwise infeasible computational and memory requirements. However, such distributed implementations often suffer from system-level issues such as slow communication and unbalanced computational nodes. The runtime of many distributed implementations are therefore throttled by that of a few slow nodes, called stragglers, or a few slow communication links, whose delays significantly encumber the overall learning task. In this paper, we propose a distributed optimization framework based on proceeding with each iteration without waiting for the stragglers, and encoding the dataset across nodes to add redundancy in the system in order to mitigate the resulting potential performance degradation due to lost updates.

%  An obvious mitigating solution is to simply continue without waiting for stragglers; e.g. using only a subset of worker node outputs at each iteration. However, in general, this can significantly degrade the performance of many optimization algorithms [citation]. In this paper, we propose a distributed optimization framework based on encoding the dataset across nodes to add redundancy in the system, eliminating this degradation for a wide class of algorithms, and thus providing  robustness against stragglers.

We consider the master-worker architecture, where the dataset is distributed across a set of worker nodes, which directly communicate to a master node to optimize a global objective. The encoding framework consists of an efficient linear transformation (coding) of the dataset that results in an overcomplete representation, which is then partitioned and distributed across the worker nodes. The distributed optimization algorithm is then performed directly on the encoded data, with all worker nodes oblivious to the encoding scheme, \emph{i.e.}, no explicit decoding of the data is performed, and nodes simply solve the effective optimization problem after encoding. In order to mitigate the effect of stragglers, in each iteration, the master node only waits for the first $k$ updates to arrive from the $m$ worker nodes (where $k\leq m$ is a design parameter) before moving on; the remaining $m - k$ node results are effectively erasures, whose loss is compensated by the data encoding.

The framework is applicable to both the data parallelism and model parallelism paradigms of distributed learning, and can be applied to distributed implementations of several popular optimization algorithms, including gradient descent, limited-memory-BFGS, proximal gradient, and block coordinate descent. We show that if the linear transformation is designed to satisfy a spectral condition resembling the restricted isometry property, the iterates resulting from the encoded version of these algorithms deterministically converge to an exact solution for the case of model paralellism, and an approximate one under data parallelism, where the approximation quality only depends on the properties of encoding and the parameter $k$. These convergence guarantees are deterministic in the sense that they hold for any pattern of node delays, \emph{i.e.}, even if an adversary chooses which nodes to delay at every iteration. In addition, the convergence behavior is independent of the tail behavior of the node delay distribution. Such a worst-case guarantee is not possible for the asynchronous versions of these algorithms, whose convergence rates deteriorate with increasing node delays. We point out that our approach is particularly suited to computing networks with a high degree of variability and unpredictability, where a large number of nodes can delay their computations for arbitrarily long periods of time.

Our contributions are as follows: {\sf (i)} We propose the encoded distributed optimization framework, and prove deterministic convergence guarantees under this framework for gradient descent, L-BFGS, proximal gradient and block coordinate descent algorithms; {\sf (ii)} we provide three classes of encoding matrices, and discuss their properties, and describe how to efficiently encode with such matrices on large-scale data; {\sf (iii)} we implement
the proposed technique on Amazon EC2 clusters and compare their performance to uncoded, replication, and asynchronous strategies for problems such as ridge regression, collaborative filtering, logistic regression, and LASSO. In these tasks we show that in the presence of stragglers, the technique can result in significant speed-ups (specific amounts depend on the underlying system, and examples are provided in Section~\ref{sec:numerical}) compared to the uncoded case when all workers are waited for in each iteration, to achieve the same test error.

\paragraph{Related work.}
The approaches to mitigating the effect of stragglers can be broadly classified into three categories: replication-based techniques, asynchronous optimization, and coding-based techniques.

Replication-based techniques consist of either re-launching a certain task if it is delayed, or pre-emptively assigning each task to multiple nodes and moving on with the copy that completes first. Such techniques have been proposed and analyzed in \cite{GardnerZbarsky_15, AnanthanarayananGhodsi_13, ShahLee_16, WangJoshi_15, YadwadkarHariharan_16}, among others. Our framework does not preclude the use of such system-level strategies, which can still be built on top of our encoded framework to add another layer of robustness against stragglers. However, it is not possible to achieve the worst-case guarantees provided by encoding with such schemes, since it is still possible for both replicas to be delayed.

Perhaps the most popular approach in distributed learning to address the straggler problem is asynchronous optimization, where each worker node asynchronously pushes updates to and fetches iterates from a parameter server independently of other workers, hence the stragglers do not hold up the entire computation. This approach was studied in \cite{RechtRe_11, AgarwalDuchi_11, DeanCorrado_12, LiAndersen_14} (among many others) for the case of data parallelism, and \cite{LiuWright_15, YouLian_16, PengXu_16, SunHannah_17} for coordinate descent methods (model parallelism). Although this approach has been largely successful, all asynchronous convergence results depend on either a hard bound on the allowable delays on the updates, or a bound on the moments of the delay distribution, and the resulting convergence rates explicitly depend on such bounds. In contrast, our framework allows for completely unbounded delays. Further, as in the case of replication, one can still consider asynchronous strategies on top of the encoding, although we do not focus on such techniques within the scope of this paper.

A more recent line of work that address the straggler problem is based on coding-theory-inspired techniques \cite{TandonLei_17, LeeLam_16, DuttaCadambe_16, KarakusSun_17, KarakusSun_17_2, YangGrover_17, HalbawiAzizanRuhi_17, ReisizadehPrakash_17}. Some of these works focus exclusively on coding for distributed linear operations, which are considerably simpler to handle. The works in \cite{TandonLei_17, HalbawiAzizanRuhi_17} propose coding techniques for distributed gradient descent that can be applied more generally. However, the approach proposed in these works require a redundancy factor of $r+1$ in the code, to mitigate $r$ stragglers. Our approach relaxes the exact gradient recovery requirement of these works, consequently reducing the amount of redundancy required by the code.

The proposed technique, especially under data parallelism, is also closely related to randomized linear algebra and sketching techniques in \cite{Mahoney_11, DrineasMahoney_11, PilanciWainwright_15}, used for dimensionality reduction of large convex optimization problems. The main difference between this literature and the proposed coding technique is that the former focuses on reducing the problem dimensions to lighten the computational load, whereas encoding \emph{increases} the dimensionality of the problem to provide robustness. As a result of the increased dimensions, coding can provide a much closer approximation to the original solution compared to sketching techniques. In addition, unlike these works, our model allows for an arbitrary convex regularizer in addition to the encoded loss term.

\section{Encoded Distributed Optimization}\label{sec:encoded}
We will use the notation $[j] = \lbp i \in \mathbb{Z}: 1\leq i \leq j\rbp$. All vector norms refer to 2-norm, and all matrix norms refer to spectral norm, unless otherwise noted. The superscript $^c$ will refer to complement of a subset, \emph{i.e.}, for $A \subseteq [m]$, $A^c = [m] \backslash A$. For a sequence of matrices $\lbp M_i \rbp$ and a set $A$ of indices, we will denote $\lb M_i\rb_{i \in A}$ to mean the matrix formed by stacking the matrices $M_i$ vertically. The main notation used throughout the paper is provided in Table~\ref{tb:notation}.

We consider a distributed computing network where the dataset $\lbp \lp x_i, y_i\rp \rbp_{i=1}^n$ is stored across a set of $m$ worker nodes, which directly communicate with a single master node. In practice the master node can be implemented using a fully-connected set of nodes, but this can still be abstracted as a single master node.

It is useful to distinguish between two paradigms of distributed learning and optimization; namely, data parallelism, where the dataset is partitioned across data samples, and model parallelism, where it is partitioned across features (see Figures~\ref{fig:data_dist} and \ref{fig:model_dist}). We will describe these two models in detail next.

\begin{table} \small
\centering
\begin{tabular}{|c|l|}
\hline
{\bf Notation} & {\bf Explanation} \\
\hline
$[j]$ & The set $\lbp i \in \mathbb{Z}: 1\leq i \leq j\rbp$ \\
$m$ & Number of worker nodes \\
$n, p$ & The dimensions of the data matrix $X \in \mathbb{R}^{n \times p}$, vector $y \in \mathbb{R}^{n \times 1}$ \\
$k_t$ & Number of updates the master node waits for in iteration $t$, before moving on \\
$\eta_t$ & Fraction of nodes waited for in iteration, \emph{i.e.}, $\eta_t = \frac{k_t}{m}$ \\
$A_t$ & The subset of nodes $[m]$ which send the fastest $k_t$ updates at iteration $t$ \\
$f(w), \wtild f(w)$ & The original and encoded objectives, respectively, under data parallelism \\
$g(w) = \phi(Xw)$ & The original objective under model parallelism \\
$\wtild g(v) = \phi(XS^\top v)$ & The encoded objective under model parallelism \\
$h(w)$ & Regularization function (potentially non-smooth) \\
$\nu$ & Strong convexity parameter \\
$L$ & Smoothness parameter for $h(w)$ (if smooth), and $g(w)$ \\
$\lambda$ & Regularization parameter \\
$\Psi_t$ & Mapping from gradient updates to step $\lbp \nabla f_i (t)\rbp_{i \in A_t} \mapsto d_t$\\
$d_t$ & Descent direction chosen by the algorithm \\
$\alpha_t$, $\alpha$ & Step size \\
$M, \mu$ & Largest and smallest eigenvalues of $X^\top X$, respectively \\
$\beta$ & Redundancy factor ($\beta \geq 1$) \\
$S$ & Encoding matrix with dimensions $\beta n \times n$ \\
$S_i$ & $i$th row-block of $S$, corresponding to worker $i$ \\
$S_A$ & Submatrix of $S$ formed by $\lbp S_i \rbp_{i \in A \subseteq [m]}$ stacked vertically \\
\hline
\end{tabular}
\caption{Notation used in the paper.}
\label{tb:notation}
\end{table}

\begin{figure}
\centering
\begin{minipage}{.46\textwidth}
  \centering
  \includegraphics[scale=0.65]{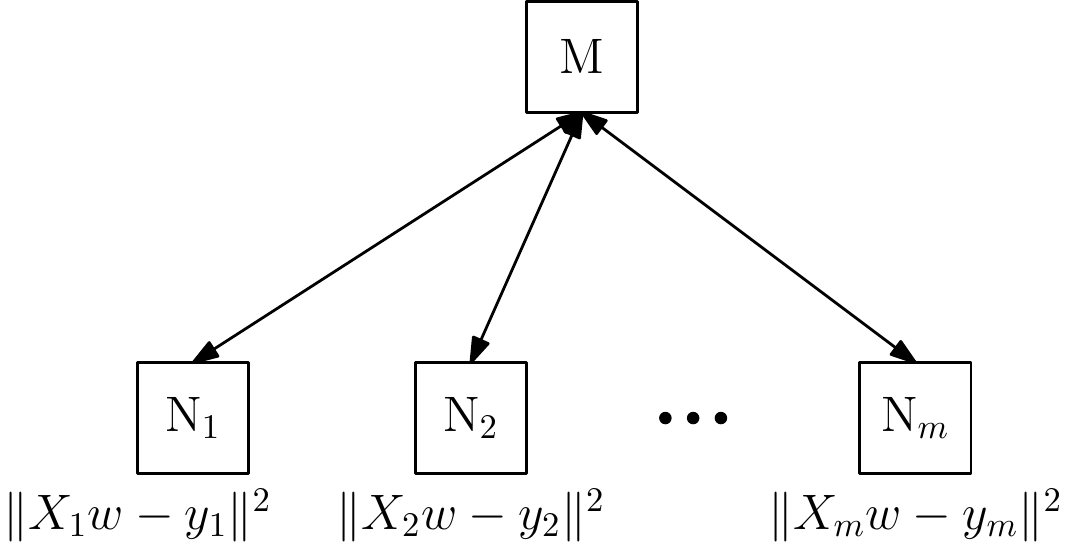}
  \caption{Uncoded distributed optimization with data parallelism, where $X$ and $y$ are partitioned as $X = \lb X_i\rb_{i \in [m]}$ and $y = \lb y_i\rb_{i \in [m]}$.}
  \label{fig:data_dist}
\end{minipage}\hfill
\begin{minipage}{.46\textwidth}
  \centering
  \includegraphics[scale=0.62]{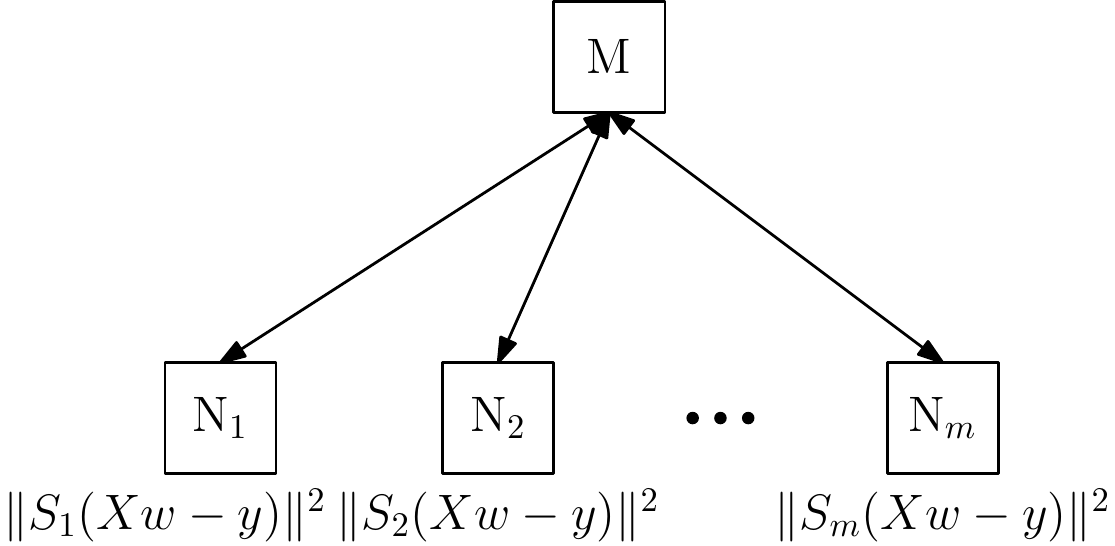}
  \caption{Encoded setup with data parallelism, where node $i$ stores $\lp S_i X, S_i y\rp$, instead of $\lp X_i, y_i\rp$. The uncoded case corresponds to $S=I$.}
  \label{fig:data_dist_enc}
\end{minipage}
\end{figure}

\subsection{Data parallelism}
We focus on objectives of the form
\begin{align}
f(w) = \frac{1}{2n} \| Xw - y\|^2 + \lambda h(w), \label{eq:original_data}
\end{align}
where $X$ and $y$ are the data matrix and data vector, respectively. We assume each row of $X$ corresponds to a data sample, and the data samples and response variables can be horizontally partitioned as $X = \lb X_1^\top \; X_2^\top \; \cdots \; X_m^\top \rb^\top$ and $y = \lb y_1^\top \; y_2^\top \; \cdots \; y_m^\top \rb^\top$. In the uncoded setting, machine $i$ stores the row-block $X_i$ (Figure~\ref{fig:data_dist}). We denote the largest and smallest eigenvalues of $X^\top X$ with $M > 0$, and $\mu \geq 0$, respectively. We assume $\lambda \geq 0$, and $h(w) \geq 0$ is a convex, extended real-valued function of $w$ that does not depend on data. Since $h(w)$ can take the value $h(w) = \infty$, this model covers arbitrary convex constraints on the optimization.

The encoding consists of solving the proxy problem
\begin{align}
\wtild f(w) = \frac{1}{2n} \| S \lp Xw - y \rp\|^2 + \lambda h(w) = \frac{1}{2n} \sum_{i=1}^m \underbrace{\| S_i \lp Xw - y \rp\|^2}_{f_i (w)} + \lambda h(w) , \label{eq:encoded_data}
\end{align}
instead, where $S \in \mathbb{R}^{\beta n \times n}$ is a designed encoding matrix with redundancy factor $\beta \geq 1$, partitioned as $S = \lb S_1^\top \; S_2^\top \; \cdots \; S_m^\top \rb^\top$ across $m$ machines. Based on this partition, worker node $i$ stores $\lp S_i X, S_i y\rp$, and operates to solve the problem \eqref{eq:encoded_data} in place of \eqref{eq:original_data} (Figure~\ref{fig:data_dist_enc}). We will denote $\hat w \in \argmin \wtild f(w)$, and $w^* \in \argmin f(w)$.

In general, the regularizer $h(w)$ can be non-smooth. We will say that $h(w)$ is $L$-smooth if $\nabla h(w)$ exists everywhere and satisfies
\begin{align*}
h(w') \leq h(w) + \langle \nabla h(w), w' - w\rangle + \frac{L}{2} \| w' - w\|^2
\end{align*}
for some $L>0$, for all $w,w'$. The objective $f$ is $\nu$-strongly convex if, for all $x,y$,
\begin{align*}
f(y) \geq f(x) + \left\langle \nabla f(x), y-x\right\rangle + \frac{\nu}{2}\| x-y\|^2.
\end{align*}

Once the encoding is done and appropriate data is stored in the nodes, the optimization process works in iterations. At iteration $t$, the master node broadcasts the current iterate $w_t$ to the worker nodes, and wait for $k_t$ gradient updates $\nabla f_i (w)$ to arrive, corresponding to that iteration, and then chooses a step direction $d_t$ and a step size $\alpha_t$ (based on algorithm $\Psi_t$ that maps the set of gradients updates to a step) to update the parameters. We will denote $\eta_t = \frac{k_t}{m}$. We will also drop the time dependence of $k$ and $\eta$ whenever it is kept constant.

The set of fastest $k_t$ nodes to send gradients for iteration $t$ will be denoted as $A_t$. Once $k_t$ updates have been collected, the remaining nodes, denoted $A_t^c$, are interrupted by the master node\footnote{If the communication is already in progress at the time when $k_t$ faster gradient updates arrive, the communication can be finished without interruption, and the late update can be dropped upon arrival. Otherwise, such interruption can be implemented by having the master node send an interrupt signal, and having one thread at each worker node keep listening for such a signal.}. Algorithms~\ref{alg:opt_data_master} and \ref{alg:opt_data_worker} describe the generic mechanism of the proposed distributed optimization scheme at the master node and a generic worker node, respectively.

The intuition behind the encoding idea is that waiting for only $k_t < m$ workers prevents the stragglers from holding up the computation, while the redundancy provided by using a tall matrix $S$ compensates for the information lost by proceeding without the updates from stragglers (the nodes in the subset $A_t^c$).

We next describe the three specific algorithms that we consider under data parallelism, to compute $d_t$.

\paragraph{Gradient descent.} In this case, we assume that $h(w)$ is $L$-smooth. Then we simply set the descent direction
\begin{align*}
	d_t = - \lp \frac{1}{2n \eta}\sum_{i \in A_t} \nabla f_i (w_t) + \lambda \nabla h(w_t) \rp.
\end{align*} 
We keep $k_t=k$ constant, chosen based on the number of stragglers in the network, or based on the desired operating regime.

\paragraph{Limited-memory-BFGS.} We assume that $h(w) = \| w\|^2$, and assume $\mu + \lambda > 0$. Although L-BFGS is traditionally a batch method, requiring updates from all nodes, its stochastic variants have also been proposed by \cite{MokhtariRibeiro_15, BerahasNocedal_16}. The key modification to ensure convergence in this case is that the Hessian estimate must be computed via  gradient components that are common in two consecutive iterations, \emph{i.e.}, from the nodes in $A_t \cap A_{t-1}$. We adapt this technique to our scenario.
For $t>0$, define $u_t := w_t - w_{t-1}$, and
\begin{align*}
r_t &: = \frac{m}{2 n\left| A_t \cap A_{t-1}\right|}\sum_{i \in A_t \cap A_{t-1}} \lp \nabla f_i(w_t) - \nabla f_i(w_{t-1})\rp. 
\end{align*}
Then once the gradient terms $\lbp \nabla f_i(w_t)\rbp_{i \in A_t}$ are collected, the descent direction is computed by $d_t = -B_t \wtild g_t$,
where $\wtild g_t = \frac{1}{2\eta n}\sum_{i \in A_t} \nabla f_i(w_t)$, and $B_t$ is the inverse Hessian estimate for iteration $t$, which is computed by
\begin{align*}
B_t^{(\ell+1)} = V_{j_{\ell,t}}^\top B_t^{(\ell)} V_{j_{\ell,t}} + \rho_{j_{\ell,t}} u_{j_{\ell, t}}u_{j_{\ell,t}}^\top, \;\;\; \rho_j = \frac{1}{r_j^\top u_j}, \;\;\; V_j = I - \rho_j r_j u_j^\top
\end{align*}
with $j_{\ell,t}= t-\wtild \sigma + \ell$, $B_t^{(0)} = \frac{r_t^\top r_t}{r_t^\top u_t} I$, and $B_t := B_t^{(\wtild \sigma)}$ with $\wtild \sigma := \min\lbp t, \sigma\rbp$, where $\sigma$ is the L-BFGS memory length. Once the descent direction $d_t$ is computed, the step size is determined through exact line search\footnote{Note that exact line search is not more expensive than backtracking line search for a quadratic loss, since it only requires a single matrix-vector multiplication.}. To do this, each worker node computes $S_i X d_t$, and sends it to the master node. Once again, the master node only waits for the fastest $k_t$ nodes, denoted by $D_t \subseteq [m]$ (where in general $D_t \neq A_t$), to compute the step size that minimizes the function along $d_t$, given by
\begin{align}
    \alpha_t = -\rho\frac{d_t^\top \wtild g_t}{d_t^\top \wtild X^\top_D \wtild X_D d_t} \label{eq:exact_ls},
\end{align}
where $\wtild X_D = \lb S_i X\rb_{i \in D_t}$, and $0<\rho<1$ is a back-off factor of choice.

%We choose $k_t = \min \lbp k: \left| A_t (k) \cap A_{t-1} \right| > \frac{m}{\beta} \rbp$. In practice, the algorithm may perform well even when this condition is not satisfied, as we will demonstrate in Section~\ref{sec:numerical}.

\paragraph{Proximal gradient. } Here, we consider the general case of non-smooth $h(w) \geq 0, \lambda \geq 0$. The descent direction $d_t$ is given by
\begin{align*}
d_t = \argmin_w \wtild F_t (w) - w_t,
\end{align*}
where 
\begin{align*}
\wtild F_t(w) &:= \frac{1}{2\eta n}\sum_{i \in A_t}  f_i (w_t) + \left\langle \frac{1}{2\eta n}\sum_{i \in A_t} \nabla f_i (w_t), w - w_t \right\rangle + \lambda h(w) + \frac{1}{2\alpha} \| w - w_t\|^2.
\end{align*}
We keep the step size $\alpha_t = \alpha$ and $k_t = k$ constant.

\begin{algorithm}
\scriptsize
\begin{algorithmic}[1]
\STATE Given: $\Psi_t$, a sequence of functions that map gradients $\lbp \nabla f_i (w_t) \rbp_{i \in A_t}$ to a descent direction $d_t$
\STATE Initialize $w_0$, $\alpha_0$
\FOR{$t = 1,\dots,T$}
	\STATE broadcast $w_t$ to all worker nodes
	\STATE wait to receive $k_t$ gradient updates $\lbp \nabla f_i (w_t) \rbp_{i \in A_t}$
	\STATE send interrupt signal the nodes in $A_t^c$
	\STATE compute the descent direction $d_t = \Psi_t \lp \lbp\nabla f_i \lp w_t\rp \rbp_{i \in A_t}\rp$
	\STATE determine step size $\alpha_t$
	\STATE take the step $w_{t+1} = w_t + \alpha_t d_t$
\ENDFOR
\end{algorithmic}
\caption{Generic encoded distributed optimization procedure under data parallelism, at the master node.}
\label{alg:opt_data_master}
\end{algorithm}

\begin{algorithm}
\scriptsize
\begin{algorithmic}[1]
\STATE Given: $f_i(w) = \| S_i (Xw - y)\|^2$
\FOR{$t = 1,\dots,T$}
	\STATE wait to receive $w_t$ \label{line:worker_standby}
	\WHILE{not interrupted by master}
		\STATE compute $\nabla f_i (w_t)$
	\ENDWHILE
	\IF{computation was interrupted}
		\STATE continue
	\ELSE
		\STATE send $\nabla f_i (w_t)$
	\ENDIF
\ENDFOR
\end{algorithmic}
\caption{Generic encoded distributed optimization procedure under data parallelism, at worker node $i$.}
\label{alg:opt_data_worker}
\end{algorithm}

\subsection{Model parallelism}
Under the model parallelism paradigm, we focus on objectives of the form
\begin{align}
\min g(w) := \min_w \phi\lp Xw\rp = \min_w \phi\lp \sum_{i=1}^m X_i w_i\rp, \label{eq:original_model}
\end{align}
where the data matrix is partitioned as $X = \lb X_1 \; X_2 \; \cdots \; X_m\rb$, the parameter vector is partitioned as $w = \lb w_1^\top \; w_2^\top \; \cdots \; w_m^\top\rb^\top$, $\phi$ is convex, and $g(w)$ is $L$-smooth. Note that the data matrix $X$ is partitioned horizontally, meaning that the dataset is split across features, instead of data samples (see Figure~\ref{fig:model_dist}). Common machine learning models, such as any regression problem with generalized linear models, support vector machine, and many other convex problems fit within this model.

We encode the problem \eqref{eq:original_model} by setting $w = S^\top v$, and solving the problem
\begin{align}
\min_v \wtild g (v) := \phi\lp XS^\top v\rp = \min_v \phi\lp \sum_{i=1}^m X S^\top_i v_i\rp, \label{eq:encoded_model}
\end{align}
where $w \in \mathbb{R}^p$ and $S^\top = \lb S_1^\top \; S_2^\top \; \cdots \; S_m^\top\rb \in \mathbb{R}^{p \times \beta p}$ (see Figure~\ref{fig:model_dist_enc}). As a result, worker $i$ stores the column-block $XS_i^\top$, as well as the iterate partition $v_i$. Note that we increase the dimensions of the parameter vector by multiplying the dataset $X$ with a wide encoding matrix $S^\top$ from the right, and as a result we have redundant coordinates in the system. As in the case of data parallelism, such redundant coordinates provide robustness against erasures arising due to stragglers. Such increase in coordinates means that the problem is simply lifted onto a larger dimensional space, while preserving the original geometry of the problem. We will denote $u_{i,t} = XS^\top_i v_{i,t}$, where $v_{i,t}$ is the parameter iterates of worker $i$ at iteration $t$. In order to compute updates to its parameters $v_i$, worker $i$ needs the up-to-date value of $\wtild z_i := \sum_{j \neq i} u_j$, which is provided by the master node at every iteration.

\begin{figure}
\centering
\begin{minipage}{.46\textwidth}
  \centering
  \includegraphics[scale=0.65]{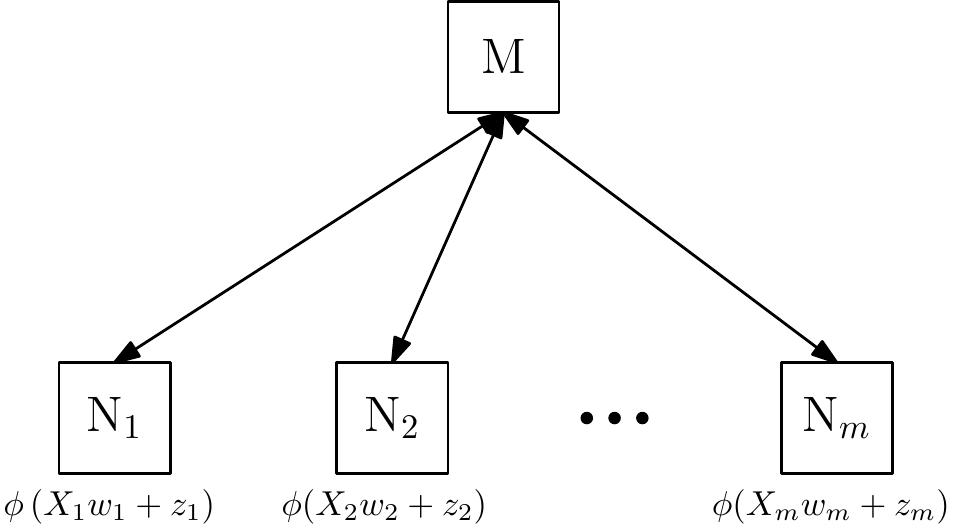}
  \caption{Uncoded distributed optimization with model parallelism, where $i$th node stores the $i$th partition of the model $w_i$. For $i=1,\dots,m$, $z_i = \sum_{j\neq i} X_j w_j$.}
  \label{fig:model_dist}
\end{minipage}\hfill
\begin{minipage}{.46\textwidth}
  \centering
  \includegraphics[scale=0.65]{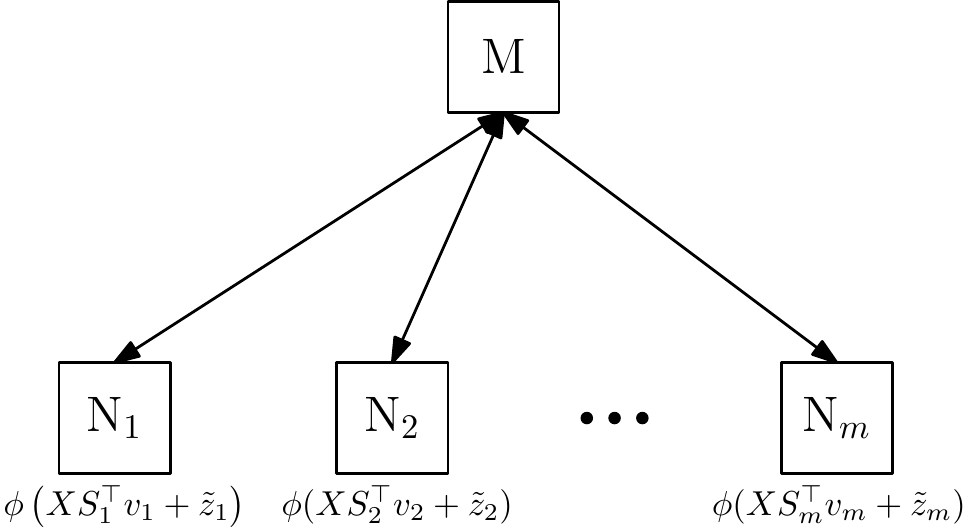}
  \caption{Encoded setup with model parallelism, where $i$th node stores the partition $v_i$ of the model in the ``lifted'' space. For $i=1,\dots,m$, $\wtild z_i = \sum_{j\neq i}  u_j = \sum_{j\neq i} X S^\top_j v_j$.}
  \label{fig:model_dist_enc}
\end{minipage}
\end{figure}

Let $\mathcal{S} = \argmin_w g(w)$, and given $w$, let $w^*$ be the projection of $w$ onto $\mathcal{S}$. We will say that $g(w)$ satisfies $\nu$-restricted-strong convexity (\cite{LaiYin_13}) if
\begin{align*}
\langle \nabla g(w), w - w^*\rangle \geq \nu \| w - w^*\|^2
\end{align*}
for all $w$. Note that this is weaker than (implied by) strong convexity since $w^*$ is restricted to be the projection of $w$, but unlike strong convexity, it is satisfied under the case where $\phi$ is strongly convex, but $X$ has a non-trivial null space, \emph{e.g.}, when it has more columns than rows.

For a given $w \in \mathbb{R}^p$, we define the level set of $g$ at $w$ as $D_g(w) := \lbp w': g(w') \leq g(w)\rbp$. We will say that the level set at $w_0$ has diameter $R$ if
\begin{align*}
\sup \lbp \| w - w'\|: w, w' \in D_g(w_0)\rbp \leq R.
\end{align*}

As in the case of data parallelism, we assume that the master node waits for $k$ updates at every iteration, and then moves onto the next iteration (see Algorithms~\ref{alg:opt_model_worker} and \ref{alg:opt_model_master}). We similarly define $A_t$ as the set of $k$ fastest nodes in iteration $t$, and also define
\begin{align*}
I_{i,t} = \lbp
\begin{array}{ll}
1 & i \in A_t  \\
0 & i \notin A_t.
\end{array}
\right.
\end{align*}

\begin{algorithm}
\scriptsize
\begin{algorithmic}[1]
\STATE Given: $X_i$, $v_i$.
\FOR{$t = 1,\dots,T$}
	\STATE wait to receive $\lp I_{i,t-1}, \wtild z_{i,t}\rp$ \label{line:worker_standby}
	\IF{ $I_{i,t} == 1$} \label{line:check_begin}
		\STATE take step $v_{i, t} = v_{i,t-1} + d_{i, t-1}$
	\ELSE
		\STATE set $v_{i, t} = v_{i,t-1}$
	\ENDIF\label{line:check_end}
	\WHILE{not interrupted by master} 
		\STATE compute next step $d_{i,t} = \alpha S_i X^\top\nabla \phi \lp XS^\top_i v_{i,t} + \wtild z_{i,t}\rp$
		\STATE compute $u_{i,t} = XS^\top_i v_{i,t}$
	\ENDWHILE
	\IF{computation was interrupted} 
		\STATE continue
	\ELSE
		\STATE send $u_{i,t}$ to master node
	\ENDIF
\ENDFOR
\end{algorithmic}
\caption{Encoded block coordinate descent at worker node $i$.}
\label{alg:opt_model_worker}
\end{algorithm}

\begin{algorithm}
\scriptsize
\begin{algorithmic}[1]
\FOR{$t = 1,\dots,T$}
	\FOR{$i = 1,\dots,m$}
		\STATE send $\lp I_{i, t-1}, \wtild z_{i,t}\rp$ to worker $i$
	\ENDFOR
	\STATE wait to receive $k$ updated parameters $\lbp u_{i,t} \rbp_{i \in A_t}$
	\STATE send interrupt signal the nodes in $A_t^c$
	\STATE set $u_{i,t} = u_{i,t-1}$ for $i \in A_t^c$
	\STATE compute $\wtild z_{i,t} = \sum_{j \neq i} u_{j,t}$ for all $i$
\ENDFOR
\end{algorithmic}
\caption{Encoded block coordinate descent at the master node.}
\label{alg:opt_model_master}
\end{algorithm}

Under model parallelism, we consider block coordinate descent, described in Algorithm~\ref{alg:opt_model_worker}, where worker $i$ stores the current values of the partition $v_i$, and performs updates on it, given the latest values of the rest of the parameters. The parameter estimate at time $t$ is denoted by $v_{i,t}$, and we also define $\wtild z_{i,t} = \sum_{j \neq i} u_{i,t} = \sum_{j \neq i} XS^\top_jv_j$. The iterates are updated by
\begin{align*}
v_{i,t} - v_{i,t-1} = \Delta_{i,t} := \lbp 
\begin{array}{cl}
-\alpha\nabla_i  \wtild g(v_{t-1}), & \text{if $i \in A_t$} \\
0, & \text{otherwise,} \\
\end{array}
\right.
\end{align*}
for a step size parameter $\alpha > 0$, where $\nabla_i$ refers to gradient only with respect to the variables $v_i$, \emph{i.e.}, $\nabla \wtild g = \lb \nabla_i \wtild g\rb_{i \in [m]}$. Note that if $i \notin A_t$ then $v_i$ does not get updated in worker $i$, which ensures the consistency of parameter values across machines. This is achieved by lines \ref{line:check_begin}--\ref{line:check_end} in Algorithm~\ref{alg:opt_model_worker}. Worker $i$ learns about this in the next iteration, when $I_{i, t-1}$ is sent by the master node. 
 
%We can express this update more compactly in the following forms:
%\begin{align*}
%\Delta_t = -\frac{\alpha}{L'} \nabla \wtild f_t(\wtild v^t) = -\frac{\alpha}{L'}P\lb
%\begin{array}{c}
%\wtild S_t^\top \nabla f(Sv^t) \\
%0
%\end{array}
%\rb,
%\end{align*}
%where $P$ is a permutation matrix that maps the indices $\lbp 1,\dots,\left| A_t\right|\rbp$ onto $A_t$.

\section{Main Theoretical Results: Convergence Analysis}\label{sec:convergence}
% condition
In this section, we prove convergence results for the algorithms described in Section~\ref{sec:encoded}. Note that since we modify the original optimization problem and solve it obliviously to this change, it is not obvious that the solution has any optimality guarantees with respect to the original problem. We show that, it is indeed possible to provide convergence guarantees in terms of the \emph{original} objective under the encoded setup.

\subsection{A spectral condition}
In order to show convergence under the proposed framework, we require the encoding matrix $S$ to satisfy a certain spectral criterion on $S$. Let $S_A$ denote the submatrix of $S$ associated with the subset of machines $A$, \emph{i.e.}, $S_A = \lb S_i\rb_{i \in A}$. Then the criterion in essence requires that for any sufficiently large subset $A$, $S_A$ behaves approximately like a matrix with orthogonal columns. We make this precise in the following statement.

\begin{definition}
Let $\beta \geq 1$, and $\frac{1}{\beta} \leq \eta \leq 1$ be given. A matrix $S \in \mathbb{R}^{\beta n \times n}$ is said to satisfy the $(m, \eta, \epsilon)$-block-restricted isometry property ($(m, \eta, \epsilon)$-BRIP) if for any $A \subseteq [m]$ with $|A| = \eta m$, 
\begin{align}
(1-\epsilon) I_n \preceq \frac{1}{\eta}S_A^\top S_A \preceq (1+\epsilon) I_n. \label{eq:rip}
\end{align}
\end{definition}
Note that this is similar to the restricted isometry property used in compressed sensing (\cite{CandesTao_05}), except that we do not require \eqref{eq:rip} to hold for every submatrix of $S$ of size $\mathbb{R}^{\eta n \times n}$. Instead, \eqref{eq:rip} needs to hold only for the submatrices of the form $S_A = \lb S_i\rb_{i \in A}$, which is a less restrictive condition. In general, it is known to be difficult to analytically prove that a structured, deterministic matrix satisfies the general RIP condition. Such difficulty extends to the BRIP condition as well. However, it is known that i.i.d. sub-Gaussian ensembles and randomized Fourier ensembles satisfy this property (\cite{CandesTao_06}). In addition, numerical evidence suggests that there are several families of constructions for $S$ whose submatrices have eigenvalues that mostly tend to concentrate around 1. We point out that although the strict BRIP condition is required for the theoretical analysis, in practice the algorithms perform well as long as the bulk of the eigenvalues of $S_A$ lie within a small interval $(1-\epsilon, 1+\epsilon)$, even though the extreme eigenvalues may lie outside of it (in the non-adversarial setting). In Section~\ref{sec:implementation}, we explore several classes of matrices and discuss their relation to this condition.

\subsection{Convergence of encoded gradient descent}
We first consider the algorithms described under data parallelism architecture. The following theorem summarizes our results on the convergence of gradient descent for the encoded problem.
\begin{theorem}\label{th:gd}
Let $w_t$ be computed using encoded gradient descent with an encoding matrix that satisfies $(m, \eta, \epsilon)$-BRIP, with step size $\alpha_t = \frac{2\zeta}{M(1+\epsilon)+L}$ for some $0 < \zeta \leq 1$, for all $t$. Let $\{ A_t\}$ be an arbitrary sequence of subsets of $[m]$ with cardinality $\left| A_t \right|\geq \eta m$ for all $t$. Then, for $f$ as given in \eqref{eq:original_data},
\begin{enumerate}
\item 
\begin{align*}
\frac{1}{t}\sum_{\tau=1}^t f(w_\tau) - \kappa_1 f(w^*) \leq \frac{ 4\epsilon f(w_0) + \frac{1}{2\alpha} \| w_{0} - w^* \|^2}{\lp 1-7\epsilon \rp t}
\end{align*}
\item If $f$ is in addition $\nu$-strongly convex, then
\begin{align*}
f(w_t) - \frac{\kappa_2^2(\kappa_2-\gamma)}{1-\kappa_2\gamma}f\lp w^*\rp \leq \lp \kappa_2 \gamma\rp^t f(w_0), \quad t=1,2,\ldots,
\end{align*}
\end{enumerate}
where $\kappa_1 = \frac{1+3\epsilon}{1-7\epsilon}$, 
$\kappa_2 = \frac{1+\epsilon}{1-\epsilon}$, and $\gamma = \lp 1 - \frac{4\nu\zeta(1-\zeta)}{M\lp 1+\epsilon\rp + L}\rp$, where $\epsilon$ is assumed to be small enough so that $\kappa_2 \gamma < 1$. 
\end{theorem}
The proof is provided in Appendix~\ref{ap:gdlbfgs},
which relies on the fact that the solution to the effective ``instantaneous" problem corresponding to the subset $A_t$ lies in a bounded set $\{ w:f(w) \leq \kappa f(w^*) \}$ (where $\kappa$ depends on the encoding matrix and strong convexity assumption on $f$), and therefore each gradient descent step attracts the iterate towards a point in this set, which must eventually converge to this set. Theorem~\ref{th:gd} shows that encoded gradient descent can achieve the standard $O\lp \frac{1}{t}\rp$ convergence rate for the general case, and linear convergence rate for the strongly convex case, up to an approximate minimum.
For the convex case, the convergence is shown on the running mean of past function values,
whereas for the strongly convex case we can bound the function value at every step. Note that although the nodes actually minimize the encoded objective $\wtild f(w)$, the convergence guarantees are given in terms of the original objective $f(w)$. 

Theorem~\ref{th:gd} provides deterministic, sample path convergence guarantees under any (adversarial) sequence of active sets $\lbp A_t\rbp$, which is in contrast to the stochastic methods, which show convergence typically in expectation. Further, the convergence rate is not affected by the tail behavior of the delay distribution, since the delayed updates of stragglers are not applied to the iterates.

Note that since we do not seek exact solutions under data parallelism, we can keep the redundancy factor $\beta$ fixed regardless of the number of stragglers. Increasing number of stragglers in the network simply results in a looser approximation of the solution, allowing for a graceful degradation. This is in contrast to existing work \cite{TandonLei_17} seeking exact convergence under coding, which shows that the redundancy factor must grow linearly with the number of stragglers.

\subsection{Convergence of encoded L-BFGS}
We consider the variant of L-BFGS described in Section~\ref{sec:encoded}. For our convergence result for L-BFGS, we need another assumption on the matrix $S$, in addition to \eqref{eq:rip}. Defining $\breve S_t = \lb S_i\rb_{i \in A_t \cap A_{t-1}}$ for $t>0$, we assume that for some $\delta>0$, 
\begin{align}
    \delta I \preceq \breve S^\top_t \breve S_t \label{eq:overlap}
\end{align}
for all $t>0$. Note that this requires that one should wait for sufficiently many nodes to send updates so that the overlap set $A_t \cap A_{t_1}$ has more than $\frac{1}{\beta}$ nodes, and thus the matrix $\breve S_t$ can be full rank. When the columns of $X$ are linearly independent, this is satisfied if $\eta \geq \frac{1}{2} + \frac{1}{2\beta}$ in the worst-case, and in the case where node delays are i.i.d. across machines, it is satisfied in expectation if $\eta \geq \frac{1}{\sqrt{\beta}}$. One can also choose $k_t$ adaptively so that $k_t = \min \lbp k: \left| A_t(k) \cap A_{t-1}\right| > \frac{1}{\beta} \rbp$. We note that although this condition is required for the theoretical analysis, the algorithm may perform well in practice even when this condition is not satisfied.

We first show that this algorithm results in stable inverse Hessian estimates under the proposed model, under arbitrary realizations of $\lbp A_t\rbp$ (of sufficiently large cardinality), which is done in the following lemma.
\begin{lemma}\label{lem:hessian_stability}
Let $\mu + \lambda > 0$. Then there exist constants $c_1, c_2 >0$ such that for all $t$, the inverse Hessian estimate $B_t$ satisfies $c_1 I \preceq B_t\preceq c_2 I$.
\end{lemma}
The proof, provided in Appendix~\ref{ap:gdlbfgs}, is based on the well-known trace-determinant method. Using Lemma~\ref{lem:hessian_stability}, we can show the following convergence result.
\begin{theorem}\label{th:lbfgs}
Let $\mu + \lambda > 0$, and let $w_t$ be computed using the L-BFGS method described in Section~\ref{sec:encoded}, with an encoding matrix that satisfies $(m, \eta, \epsilon)$-BRIP. Let $\{ A_t\}, \{D_t\}$ be arbitrary sequences of subsets of $[m]$ with cardinality $\left| A_t \right|, \left| D_t \right| \geq \eta m$ for all $t$. Then, for $f$ as described in Section~\ref{sec:encoded},
\begin{align*}
f(w_t) - \frac{\kappa^2(\kappa-\gamma)}{1-\kappa\gamma}f\lp w^*\rp \leq \lp \kappa \gamma\rp^t f(w_0),
\end{align*}
where $\kappa = \frac{1+\epsilon}{1-\epsilon}$, and $\gamma = \lp 1-\frac{4(\mu+\lambda) c_1 c_2}{\lp M+\lambda\rp(1+\epsilon) \lp c_1+c_2\rp^2}\rp$, where $c_1$ and $c_2$ are the constants in Lemma~\ref{lem:hessian_stability}. 
\end{theorem}
Similar to Theorem~\ref{th:gd}, the proof is based on the observation that the solution of the effective problem at time $t$ lies in a bounded set around the true solution $w^*$. As in gradient descent, coding enables linear convergence deterministically, unlike the stochastic and multi-batch variants of L-BFGS, \emph{e.g.}, \cite{MokhtariRibeiro_15, BerahasNocedal_16}.

\subsection{Convergence of encoded proximal gradient}
Next we consider the encoded proximal gradient algorithm, described in Section~\ref{sec:encoded}, for objectives with potentially non-smooth regularizers $h(w)$. The following theorem characterizes our convergence results under this setup.
\begin{theorem}\label{th:prox}
Let $w_t$ be computed using encoded proximal gradient with an encoding matrix that satisfies $(m, \eta, \epsilon)$-BRIP, with step size $\alpha_t =\alpha < \frac{1}{M}$, and where $\epsilon < \frac{1}{7}$. Let $\{ A_t\}$ be an arbitrary sequence of subsets of $[m]$ with cardinality $\left| A_t \right|\geq \eta m$ for all $t$. Then, for $f$ as described in Section~\ref{sec:encoded},
\begin{enumerate}
\item For all $t$,
\begin{align*}
\frac{1}{t}\sum_{\tau=1}^t f(w_\tau) - \kappa f(w^*) \leq \frac{ 4\epsilon f(w_0) + \frac{1}{2\alpha} \| w_{0} - w^* \|^2}{\lp 1-7\epsilon \rp t},
\end{align*}
\item For all $t$,
\begin{align*}
f(w_{t+1}) \leq \kappa  f(w_t),
\end{align*}
where $\kappa = \frac{1+7\epsilon}{1-3\epsilon}$.
\end{enumerate}
\end{theorem}
As in the previous algorithms, the convergence guarantees hold for arbitrary sequences of active sets $\lbp A_t \rbp$. Note that as in the gradient descent case, the convergence is shown on the mean of past function values. Since this does not prevent the iterates from having a sudden jump at a given iterate, we include the second part of the theorem to complement the main convergence result, which implies that the function value cannot increase by more than a small factor of its current value.

\subsection{Convergence of encoded block coordinate descent}

Finally, we consider the convergence of encoded block coordinate descent algorithm. The following theorem characterizes our main convergence result for this case.
\begin{theorem}\label{th:bcd}
Let $w_t = S^\top v_t$, where $v_t$ is computed using encoded block coordinate descent as described in Section~\ref{sec:encoded}. Let $S$ satisfy $(m, \eta, \epsilon)$-BRIP, and the step size satisfy $\alpha < \frac{1}{L(1+\epsilon)}$. Let $\{ A_t\}$ be an arbitrary sequence of subsets of $[m]$ with cardinality $\left| A_t \right|\geq \eta m$ for all $t$. Let the level set of $g$ at the first iterate $D_g(w_0)$ have diameter $R$. Then, for $g(w) = \phi(Xw)$ as described in Section~\ref{sec:encoded}, the following hold.
\begin{enumerate}
\item If $\phi$ is convex, then
\begin{align*}
g(w_t) - g(w^*) \leq \frac{1}{\frac{1}{\pi_0} + Ct},
\end{align*}
where $\pi_0 = g(w_0) - g(w^*)$, and $C =\frac{(1-\epsilon)\alpha}{R}\lp 1 - \frac{\alpha L'}{2}\rp$.
\item If $g$ is $\nu$-restricted-strongly convex, then
\begin{align*}
g(w_t) - g(w^*) \leq \lp 1 - \frac{1}{\xi}\rp^t \lp g(w_0) - g(w^*) \rp,
\end{align*}
where $\xi = \frac{1}{\nu (1-\epsilon) \alpha} \lp 1 - \frac{L(1+\epsilon) \alpha}{2}\rp^{-1}$.
\end{enumerate}
\end{theorem}
Theorem~\ref{th:bcd} demonstrates that the standard $O\lp \frac{1}{t}\rp$ rate for the general convex, and linear rate for the strongly convex case can be obtained under the encoded setup. Note that unlike the data parallelism setup, we can achieve exact minimum under model parallelism, since the underlying geometry of the problem does not change under encoding; the same objective is simply mapped onto a higher-dimensional space, which has redundant coordinates. Similar to the previous cases, encoding allows for deterministic convergence guarantees under adversarial failure patterns. This comes at the expense of a small penalty in the convergence rate though; one can observe that a non-zero $\epsilon$ slightly weakens the constants in the convergence expressions. Still, note that this penalty in convergence rate only depends on the encoding matrix and not on the delay profile in the system. This is in contrast to the asynchronous coordinate descent methods; for instance, in \cite{LiuWright_15}, the step size is required to shrink \emph{exponentially} in the maximum allowable delay, and thus the guaranteed convergence rate can exponentially degrade with increasing worst-case delay in the system. The same is true for the linear convergence guarantee in \cite{PengXu_16}.

\section{Code Design}\label{sec:implementation}
% code design with regard to rip [etf + random]
\subsection{Block RIP condition and code design}

We first discuss two classes of encoding matrices with regard to the BRIP condition; namely equiangular tight frames, and random matrices.

\paragraph{Tight frames.} A unit-norm \emph{frame} for $\mathbb{R}^n$ is a set of vectors $F = \lbp a_i \rbp_{i=1}^{n\beta}$ with $\|a_i\|=1$, where $\beta \geq 1$, such that there exist constants $\xi_2 \geq \xi_1>0$ such that, for any $u \in \mathbb{R}^n$, 
\begin{align*}
    \xi_1 \| u\|^2 \leq \sum_{i=1}^{n\beta} \left| \langle u, a_i\rangle\right|^2\leq \xi_2 \| u\|^2.
\end{align*}
The frame is  \emph{tight} if the above satisfied with $\xi_1 = \xi_2$. In this case, it can be shown that the constants are equal to the redundancy factor of the frame, \emph{i.e.}, $\xi_1=\xi_2=\beta$. If we form $S \in \mathbb{R}^{(\beta n) \times n}$ by rows that form a \emph{tight frame}, then we have $S^\top S = \beta I$, which ensures $\| Xw - y\|^2=\tfrac{1}{\beta}\| SXw - Sy\|^2$. 
Then for any solution $\hat w$ to the encoded problem (with $k=m$),  
\begin{align*}
\nabla \wtild f(\hat w) = X^\top S^\top S (X \hat w-y) = \beta X^\top (X\hat w-y) = \beta \nabla f( \hat w).
\end{align*}
Therefore, the solution to the encoded problem satisfies the optimality condition for the original problem as well:
\begin{align*}
     -\nabla \wtild f(\hat w) \in \partial h(\hat w), \quad\Leftrightarrow \quad -\nabla f(\hat w) \in \partial h(\hat w), 
\end{align*}
and if $f$ is also strongly convex, then $\hat w = w^*$ is the unique solution. This means that for $k=m$, obliviously solving the encoded problem results in the same objective value as in the original problem.

Define the maximal inner product of a unit-norm tight frame $F = \lbp a_i\rbp_{i=1}^{n\beta}$, where $a_i \in \mathbb{R}^n, \forall i$, by
\begin{align*}
\omega(F) := \max_{\substack{a_i, a_j \in F \\ i \neq j}} \left| \langle a_i, a_j\rangle \right|.
\end{align*}
A tight frame is called an \emph{equiangular tight frame} (ETF) if $\left| \langle a_i, a_j\rangle \right| = \omega(F)$ for every $i \neq j$. 
\begin{proposition}[\cite{Welch_74}]\label{prop:welch}
Let $F = \lbp a_i\rbp_{i=1}^{n\beta}$ be a tight frame. Then $\omega(F) \geq \sqrt{\frac{\beta-1}{n\beta-1}}$. Moreover, equality is satisfied if and only if $F$ is an equiangular tight frame.
\end{proposition}
Therefore, an ETF minimizes the correlation between its individual elements, making each submatrix $ S_A^\top S_A$ as close to orthogonal as possible. This, combined with the property that tight frames preserve the optimality condition when all nodes are waited for ($k=m$), make ETFs good candidates for encoding, in light of the required property \eqref{eq:rip}. We specifically evaluate the Paley ETF from \cite{Paley_33} and \cite{GoethalsSeidel_67}; Hadamard ETF from \cite{Szollosi_13} (not to be confused with Hadamard matrix); and Steiner ETF from \cite{FickusMixon_12} in our experiments.

\begin{figure}
\centering
\begin{minipage}{.46\textwidth}
  \centering
  \includegraphics[scale=0.42]{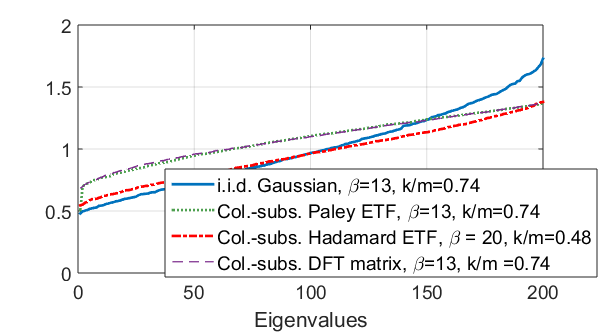}
  \caption{Sample spectrum of $S_A^\top S_A$ for various constructions with high redundancy, and small $k$ (normalized).}
  \label{fig:high_red}
\end{minipage}\hfill
\begin{minipage}{.46\textwidth}
  \centering
  \includegraphics[scale=0.42]{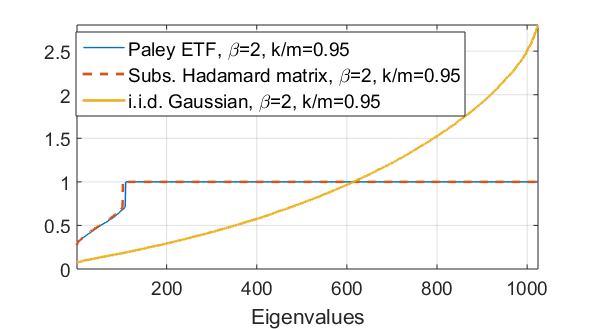}
  \caption{Sample spectrum of $S_A^\top S_A$ for various constructions with moderate redundancy, and large $k$ (normalized).}
  \label{fig:low_red}
\end{minipage}
\end{figure}

Although the derivation of tight eigenvalue bounds for subsampled ETFs is a long-standing problem, numerical evidence (see Figures~\ref{fig:high_red}, \ref{fig:low_red}) suggests that they tend to have their eigenvalues more tightly concentrated around 1 than random matrices (also supported by the fact that they satisfy Welch bound, Proposition~\ref{prop:welch} with equality). 

Note that 
%even though 
our theoretical results focus on the extreme eigenvalues due to a worst-case analysis; in practice, 
%the distribution of the bulk of the eigenvalues have a larger impact on the final accuracy of the solution, since 
most of the energy of the gradient lies on the eigen-space associated with the bulk of the eigenvalues, which the following proposition shows can be identically 1.
\begin{proposition}\label{prop:etf_bulk}
If the rows of $S$ are chosen to form an ETF with redundancy $\beta$, then for $\eta \geq 1- \frac{1}{\beta}$, $\frac{1}{\beta}S_A^\top S_A$ has $n(1-\beta(1-\eta))$ eigenvalues equal to 1.
\end{proposition}
This follows immediately from Cauchy interlacing theorem, using the fact that $S_AS_A^\top$ and $S_A^\top S_A$ have the same spectra except zeros. Therefore for sufficiently large $\eta$, ETFs have a mostly flat spectrum even for low redundancy, and thus in practice one would expect ETFs to perform well even for small amounts of redundancy. This is also confirmed by Figure~\ref{fig:low_red}, as well as our numerical results.

\paragraph{Random matrices.} Another natural choice of encoding could be to use i.i.d. random matrices. Although encoding with such random matrices can be computationally expensive and may not have the desirable properties of encoding with tight frames, their eigenvalue behavior can be characterized analytically. In particular, using the existing results on the eigenvalue scaling of large i.i.d. Gaussian matrices from \cite{Geman_80, Silverstein_85} and union bound, it can be shown that
\begin{align}
    &\Prob{\max_{A:\left|A\right|=k}\lambda_{\max}\lp \frac{1}{\beta \eta n} S_A^\top S_A\rp > \lp 1 + \sqrt{\frac{1}{\beta \eta}} \rp^2} \to 0 \label{eq:random1}\\& \Prob{\min_{A:\left|A\right|=k}\lambda_{\min}\lp \frac{1}{\beta \eta n} S_A^\top S_A\rp < \lp 1 - \sqrt{\frac{1}{\beta \eta}} \rp^2} \to 0, \label{eq:random2}
\end{align}
as $n \to \infty$, if the elements of $S_A$ are drawn i.i.d. from $N(0,1)$. Hence, for sufficiently large redundancy and problem dimension, i.i.d. random matrices are good candidates for encoding as well. However, for finite $\beta$, even if $k=m$, in general the optimum of the original problem is not recovered exactly, for such matrices.

\subsection{Efficient encoding}

In this section we discuss some of the possible practical approaches to encoding. Some of the practical issues involving encoding include the the computational complexity of encoding, as well as the loss of sparsity in the data due to the multiplication with $S$, and the resulting increase in time and space complexity. We address these issues in this section. 

\subsubsection{Efficient distributed encoding with sparse matrices} \label{subsec:sparse}
% setup, definitions
Let the dataset $(X,y)$ lie in a database, accessible to each worker node, where each node is responsible for computing their own encoded partitions $S_iX$ and $S_i y$. We assume that $S$ has a sparse structure. Given $S$, define $B_i (S) = \lbp j: S_{ij} \neq 0\rbp$ as the set of indices of the non-zero elements of the $i$th row of $S$. For a set $\mathcal{I}$ of rows, we define $B_{\mathcal{I}}(S) = \cup_{i \in \mathcal{I}} B_i(S)$. 

% partitioning of indices, storage indices
Let us partition the set of rows of $S$, $[\beta n]$, into $m$ machines, and denote the partition of machine $k$ as $\mathcal{I}_k$, \emph{i.e.}, $\bigsqcup_{k=1}^m \mathcal{I}_k = [\beta n]$, where $\sqcup$ denotes disjoint union. Then the set of non-zero columns of $S_k$ is given by $B_{\mathcal{I}_k} (S)$. Note that in order to compute $S_k X$, machine $k$ only requires the rows of $X$ in the set $B_{\mathcal{I}_k} (S)$. In what follows, we will denote this submatrix of $X$ by $\wtild X_k$, \emph{i.e.}, if $x_i^\top$ is the $i$th row of $X$, $\wtild X_k := \lb x_i^\top\rb_{i \in B_{\mathcal{I}_k} (S)}$. Similarly $\wtild y_k = \lb y_i\rb_{i \in B_{\mathcal{I}_k} (S)}$, where $y_i$ is the $i$th element of $y$.

Consider the specific computation that needs to be done by worker $k$ during the iterations, for each algorithm. Under the data parallelism setting, worker $k$ computes the following gradient:
\begin{align}
\nabla f_k(w) = X^\top S_k^\top S_k (Xw - y) \overset{\aaaa}{=} \wtild X_k^\top S_k^\top S_k (\wtild X_k w - \wtild y_k) \label{eq:gradient_data}
\end{align}
where (a) follows since the rows of $X$ that are not in $B_{\mathcal{I}_k}$ get multiplied by zero vector. Note that the last expression can be computed without any matrix-matrix multiplication. This gives a natural storage and computation scheme for the workers. Instead of computing $S_k X$ offline and storing it, which can result in a loss of sparsity in the data, worker $k$ can store $\wtild X_k$ in uncoded form, and compute the gradient through \eqref{eq:gradient_data} whenever needed, using only matrix-vector multiplications. Since $S_k$ is sparse, the overhead associated with multiplications of the form $S_k v$ and $S_k^\top v$ is small.

Similarly, under model parallelism, the computation required by worker $k$ is
\begin{align}
\nabla_k \wtild g(v) = S_k X^\top \nabla_k \phi \lp XS^\top_k v_k + \wtild z_k\rp = S_k \wtild X_k^\top \nabla_k \phi \lp \wtild X_k S^\top_k v_k + \wtild z_k\rp, \label{eq:gradient_model}
\end{align}
and as in the data parallelism case, the worker can store $\wtild X_k$ uncoded, and compute \eqref{eq:gradient_model} online through matrix-vector multiplications. 

\paragraph{Example: Steiner ETF.}  We illustrate the described technique through Steiner ETF, based on the construction proposed in \cite{FickusMixon_12}, using $(2,2,v)$-Steiner systems. Let $v$ be a power of 2, let $H \in \mathbb{R}^{v \times v}$ be a real Hadamard matrix, and let $h_i$ be the $i$th column of $H$, for $i=1,\dots,v$. Consider the matrix $V \in \lbp 0,1\rbp^{v \times v(v-1)/2}$, where each column is the incidence vector of a distinct two-element subset of $\lbp 1,\dots,v\rbp$. For instance, for $v=4$,
\begin{align*}
V = \lb
\begin{array}{cccccc}
1&1&1&0&0&0\\
1&0&0&1&1&0\\
0&1&0&1&0&1\\
0&0&1&0&1&1
\end{array}
\rb.
\end{align*}
Note that each of the $v$ rows have exactly $v-1$ non-zero elements. We construct Steiner ETF $S$ as a $v^2 \times \frac{v(v-1)}{2}$ matrix by replacing each 1 in a row with a distinct column of $H$, and normalizing by $\sqrt{v-1}$. For instance, for the above example, we have
\begin{align*}
S = \frac{1}{\sqrt{3}} \lb 
\begin{array}{cccccc}
h_2&h_3&h_4&0&0&0\\
h_2&0&0&h_3&h_4&0\\
0&h_2&0&h_3&0&h_4\\
0&0&h_2&0&h_3&h_4
\end{array}
\rb.
\end{align*}
We will call a set of rows of $S$ that arises from the same row of $V$ a block. In general, this procedure results in a matrix $S$ with redundancy factor $\beta = \frac{2v}{v-1}$. In full generality, Steiner ETFs can be constructed for larger redundancy levels; we refer the reader to \cite{FickusMixon_12} for a full discussion of these constructions. 

We partition the rows of the $V$ matrix into $m$ machines, so that each machine gets assigned $\frac{v}{m}$ rows of $V$, and thus the corresponding $\frac{v}{m}$ blocks of $S$. %Let us denote the block indices that get assigned to node $k$ with $\mathcal{J}_k$, and the $j$th block at the $k$th machine with $S_k^{(j)}$, so that $S_k = \lb S_k^{(j)}\rb_{j \in \mathcal{J}_k}$.

%Recall that the encoding matrix is partitioned as $S = \lb S_1^\top \; \dots \; S_m^\top\rb^\top$ across the machines. Let us further decompose each block $S_k$ as $S_k = \lb S_k^{(1)} \; \dots \; S_k^{\lp \frac{v}{m}\rp}\rb$, so that the sub-block $S_k^{(j)}$ corresponds to 
This construction and partitioning scheme is particularly attractive for our purposes for two reasons. First, it is easy to see that for any node $k$, $\left| B_{\mathcal{I}_k}\right|$ is upper bounded by $\frac{v(v-1)}{m} = \frac{2n}{m}$, which means the memory overhead compared to the uncoded case is limited to a factor\footnote{In practice, we have observed that the convergence performance improves when the blocks are broken into multiple machines, so one can, for instance, assign half-blocks to each machine.} of $\beta$. Second, each block of $S_k$ consists of (almost) a Hadamard matrix, so the multiplication $S_k v$ can be efficiently implemented through Fast Walsh-Hadamard Transform.

\paragraph{Example: Haar matrix.} Another possible choice of sparse matrix is column-subsampled Haar matrix, which is defined recursively by
\begin{align*}
H_{2n} = \frac{1}{\sqrt{2}}\lb \begin{array}{c} H_{n} \otimes \lb 1 \; 1\rb \\  I_{n} \otimes \lb 1 \; -1\rb\end{array}\rb, \;\; H_1 = 1,
\end{align*}
where $\otimes$ denotes Kronecker product. Given a redundancy level $\beta$, one can obtain $S$ by randomly sampling $\frac{n}{\beta}$ columns of $H_n$. It can be shown that in this case, we have $\left| B_{\mathcal{I}_k}\right| \leq \frac{\beta n \log(n)}{m}$, and hence encoding with Haar matrix incurs a memory cost by logarithmic factor.

\subsubsection{Fast transforms}\label{subsec:encdist2:efficient:fast}

Another computationally efficient method for encoding is to use fast transforms: Fast Fourier Transform (FFT), if $S$ is chosen as a subsampled DFT matrix, and the Fast Walsh-Hadamard Transform (FWHT), if $S$ is chosen as a subsampled real Hadamard matrix. In particular, one can insert rows of zeroes at random locations into the data pair $(X,y)$, and then take the FFT or FWHT of each column of the augmented matrix. This is equivalent to a randomized Fourier or Hadamard ensemble, which is known to satisfy the RIP with high probability by \cite{CandesTao_06}. However, such transforms do not have the memory advantages of the sparse matrices, and thus they are more useful for the setting where the dataset is dense, and the encoding is done offline.

\subsection{Cost of encoding}\label{subsec:encdist2:efficient:cost}
Since encoding increases the problem dimensions, it clearly comes with the cost of increased space complexity. The memory and storage requirement of the optimization still increases by a factor of 2, if the encoding is done offline (for dense datasets), or if the techniques described in the previous subsection are applied (for sparse datasets)\footnote{Note that the increase in space complexity is not higher for sparse matrices, since the sparsity loss can be avoided using the technique described in Section~\ref{subsec:sparse}}. Note that the added redundancy can come by increasing the amount of effective data points per machine, by increasing the number of machines while keeping the load per machine constant, or a combination of the two. In the first case, the computational load per machine increases by a factor of $\beta$. Although this can make a difference if the system is bottlenecked by the computation time, distributed computing systems are typically communication-limited, and thus we do not expect this additional cost to dominate the speed-up from the mitigation of stragglers.

%One can also consider Haar transform as a method of efficient encoding. Haar transform has two possible advantages. First, it allows linear time encoding, unlike FFT and FWHT. Second, due to its sparsity, it enables the efficient distributed implementation described in the previous subsection. We will explore the performance of Haar matrix for encoding in the next section.

%\subsection{System-level design}
%In this section, we explain the communication and interruption procedures in slightly more detail. Although there can be many ways to address such system-level issues 

% keep brief, don't go into detail. just clarification on interruption (when and how?).

\section{Numerical Results}\label{sec:numerical}

We implement the proposed technique on four problems: ridge regression, matrix factorization, logistic regression, and LASSO.

\subsection{Ridge regression}
We generate the elements of matrix $X$ i.i.d. $\sim N(0,1)$, and the elements of $y$ are generated from $X$ and an i.i.d. $N(0,1)$ parameter vector $w^*$, through a linear model with Gaussian noise, for dimensions $(n,p)=(4096, 6000)$. We solve the problem $\min_w \frac{1}{2 n}\left\| S \lp Xw - y\rp \right\|^2 + \frac{\lambda}{2} \| w\|^2$,
for regularization parameter $\lambda=0.05$. We evaluate column-subsampled Hadamard matrix with redundancy $\beta=2$ (encoded using FWHT), replication and uncoded schemes. We implement distributed L-BFGS as described in Section~\ref{sec:convergence} on an Amazon EC2 cluster using \texttt{mpi4py} Python package, over $m=32$ \texttt{m1.small} instances as worker nodes, and a single \texttt{c3.8xlarge} instance as the central server.

\begin{figure}
\centering
  \includegraphics[scale=0.38]{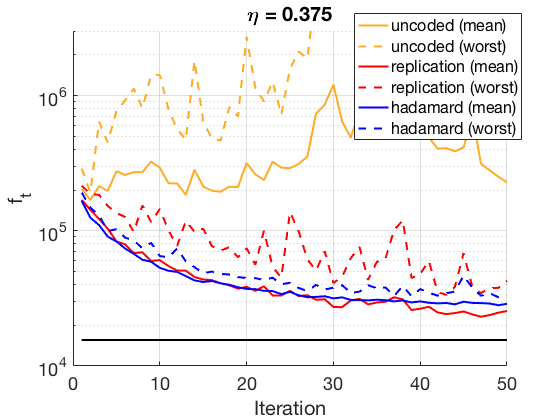}
  \includegraphics[scale=0.38]{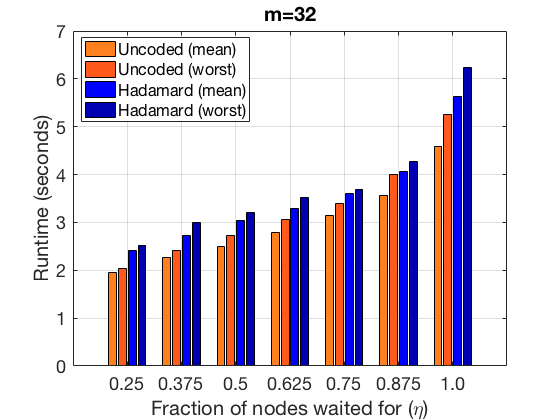}
  \caption{{\bf Left:} Sample evolution of uncoded, replication, and Hadamard (FWHT)-coded cases, for $k=12$, $m=32$. {\bf Right:} Runtimes of the schemes for different values of $\eta$, for the same number of iterations for each scheme. Note that this essentially captures the delay profile of the network, and does not reflect the relative convergence rates of different methods.}
  \label{fig:ec2}
\end{figure}

Figure~\ref{fig:ec2} shows the result of our experiments, which are aggregated from 20 trials. In addition to uncoded scheme, we consider data replication, where each uncoded partition is replicated $\beta=2$ times across nodes, and the server discards the duplicate copies of a partition, if received in an iteration. It can be seen that for low $\eta$, uncoded L-BFGS may not converge when a fixed number of nodes are waited for, whereas the Hadamard-coded case stably converges. We also observe that the data replication scheme converges on average, but its performance may deteriorate if both copies of a partition are delayed. Figure~\ref{fig:ec2} suggests that this performance can be achieved with an approximately $40\%$ reduction in the runtime, compared to waiting for all the nodes.

\begin{figure}
\centering
\begin{minipage}{.64\textwidth}
  \centering
    \includegraphics[width=1.85in]{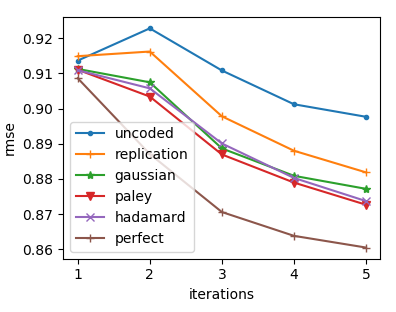}
    \includegraphics[width=1.85in]{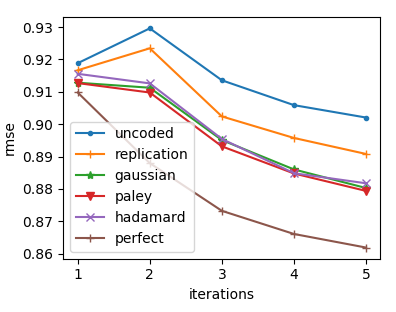}
    \includegraphics[width=1.85in]{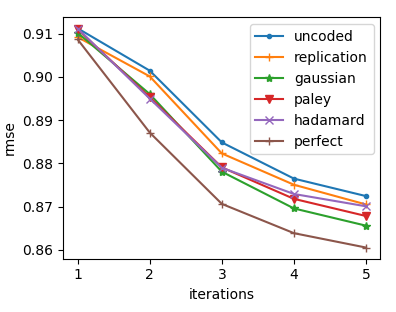}
    \includegraphics[width=1.85in]{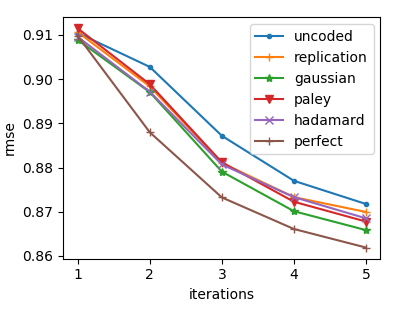}
  \caption{Test RMSE for $m = 8$ (left) and $m = 24$ (right) nodes, where the server waits for $k = m / 8$ (top) and $k = m/2$ (bottom) responses. ``Perfect" refers to the case where $k=m$.}
  \label{fig:movielens_perf}
\end{minipage}\hfill
\begin{minipage}{.33\textwidth}
\vfill
  \centering
    \includegraphics[width=1.8in]{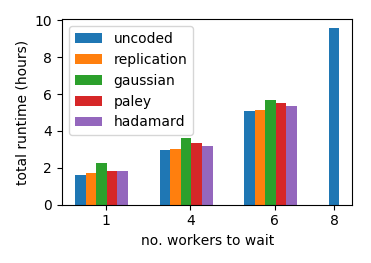}
    \includegraphics[width=1.8in]{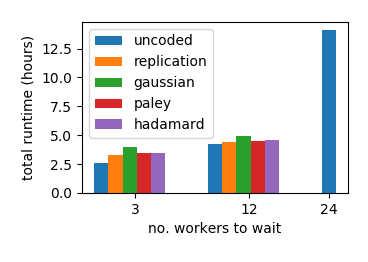}
    \caption{Total runtime with $m=8$ and $m=24$ nodes for different values of $k$, under fixed 100 iterations for each scheme.}
  \label{fig:movielens_time}
\end{minipage}
\end{figure}

\subsection{Matrix factorization}
We next apply matrix factorization on the MovieLens-1M dataset (\cite{RiedlKonstan_98}) for the movie recommendation task. We are given $R$, a sparse matrix of movie ratings 1--5, of dimension $\# users \times \# movies$, where $R_{ij}$ is specified if user $i$ has rated movie $j$. We withhold randomly 20\% of these ratings to form an 80/20 train/test split. 
The goal is to recover user vectors $x_i\in \mathbb{R}^p$ and movie vectors $y_i\in \mathbb{R}^p$ (where $p$ is the embedding dimension) such that $R_{ij} \approx  x_i^Ty_j + u_i + v_j +b$, where $u_i$, $v_j$, and $b$ are user, movie, and global biases, respectively. The optimization problem is given by
\begin{equation}
\min_{x_i, y_j, u_i, v_j} \sum_{i,j \text{: observed}}(R_{ij} - u_i - v_j - x_i^Ty_j-b)^2 + \lambda \left(\sum_i \|x_i\|_2^2 +\|u\|_2^2 + \sum_j\|y_j\|_2^2 + \|v\|_2^2\right).
\label{eq:movielens_full}
\end{equation}
We choose $b = 3$, $p = 15$, and $\lambda = 10$, which achieves test RMSE 0.861, close to the current best test RMSE on this dataset using matrix factorization\footnote{\texttt{http://www.mymedialite.net/examples/datasets.html}}. 

Problem \eqref{eq:movielens_full} is often solved using alternating minimization, minimizing first over all $\lp x_i, u_i\rp$, and then all $\lp y_j, v_j\rp$, in repetition. Each such step further decomposes by row and column, made smaller by the sparsity of $R$. To solve for $\lp x_i, u_i\rp$, we first extract $I_i = \{j \mid r_{ij} \text{ is observed}\}$, and minimize 
\begin{equation}
\left(\bmat y_{I_i}^T, \mb 1 \emat
\bmat x_{i} \\ u_i \emat -
(R_{i,I_i}^T - v_{I_i} - b \mb 1)\right)^2 + \lambda \lp\sum_i \|x_i\|_2^2 +\|u\|_2^2\rp
\label{eq:movielens_step}
\end{equation}
for each $i$, which gives a sequence of regularized least squares problems with variable $w = [x_i^T, u_i]^T$, which we solve distributedly using coded L-BFGS; and repeat for $w = [y_j^T, v_j]^T$, for all $j$.

The Movielens experiment is run on a single 32-core machine
with Linux 4.4.
In order to simulate network latency, an artificial delay of $\Delta \sim \text{exp}(\text{10 ms})$ is imposed each time the worker completes a task. Small problem instances ($n< 500$) are solved locally at the central server, using the built-in function \texttt{numpy.linalg.solve}.
To reduce overhead, we create a bank of encoding matrices $\lbp S_n\rbp$ for Paley ETF and Hadamard ETF, for $n=100, 200, \hdots, 3500$, and then given a problem instance, subsample the columns of the appropriate matrix $S_n$ to match the dimensions.
Overall, we observe that encoding overhead is amortized by the speed-up of the distributed optimization.

Figure \ref{fig:movielens_perf} gives the final performance of our distributed L-BFGS for various encoding schemes, for each of the 5 epochs, which shows that coded schemes are most robust for small $k$.
A full table of results is given in Appendix~\ref{ap:movielens}.

\subsection{Logistic regression}
In our next experiment, we apply logistic regression for document classification for Reuters Corpus Volume 1 (\texttt{rcv1.binary}) dataset  from \cite{LewisYang_04}, where we consider the binary task of classifying the documents into corporate/industrial/economics vs. government/social/markets topics. The dataset has 697,641 documents, and 47,250 term frequency-inverse document frequency (tf-idf) features. We randomly select 32,500 features for the experiment, and reserve 100,000 documents for the test set. We use logistic regression with $\ell_2$-regularization for the classification task, with the objective
\begin{align*}
\min_{w,b} \frac{1}{n} \sum_{i=1}^n \log \lp 1 + \exp\lbp -z_i^\top w + b\rbp\rp + \lambda \| w\|^2,
\end{align*}
where $z_i = y_i x_i$ is the data sample $x_i$ multiplied by the label $y_i \in \lbp -1,1\rbp$, and $b$ is the bias variable. We solve this optimization using encoded distributed block coordinate descent as described in Section~\ref{sec:encoded}, and implement Steiner and Haar encoding as described in Section~\ref{sec:implementation}, with redundancy $\beta = 2$. In addition we implement the asynchronous coordinate descent, as well as replication, which represents the case where each partition $Z_i$ is replicated across two nodes, and the faster copy is used in each iteration. We use $m=128$ \texttt{t2.medium} instances as worker nodes, and a single \texttt{c3.4xlarge} instance as the master node, which communicate using the \texttt{mpi4py} package. We consider two models for stragglers. In the first model, at each node, we add a random delay drawn from a Gaussian mixture distribution $q \mathcal{N}(\mu_1, \sigma_1^2) + (1-q)\mathcal{N}(\mu_2, \sigma_2^2)$, where $q=0.5$, $\mu_1 = 0.5$s, $\mu_2 = 20$s, $\sigma_1=0.2$s, $\sigma_2=5$s. In the second model, we do not directly add any delay, but at each machine we launch a number of dummy background tasks (matrix multiplication) that are executed throughout the computation. The number of background tasks across the nodes is distributed according to a power law with exponent $\alpha = 1.5$. The number of background tasks launched is capped at 50.

\begin{figure}
\centering
\begin{minipage}{.48\textwidth}
  \centering
    \includegraphics[width=3in]{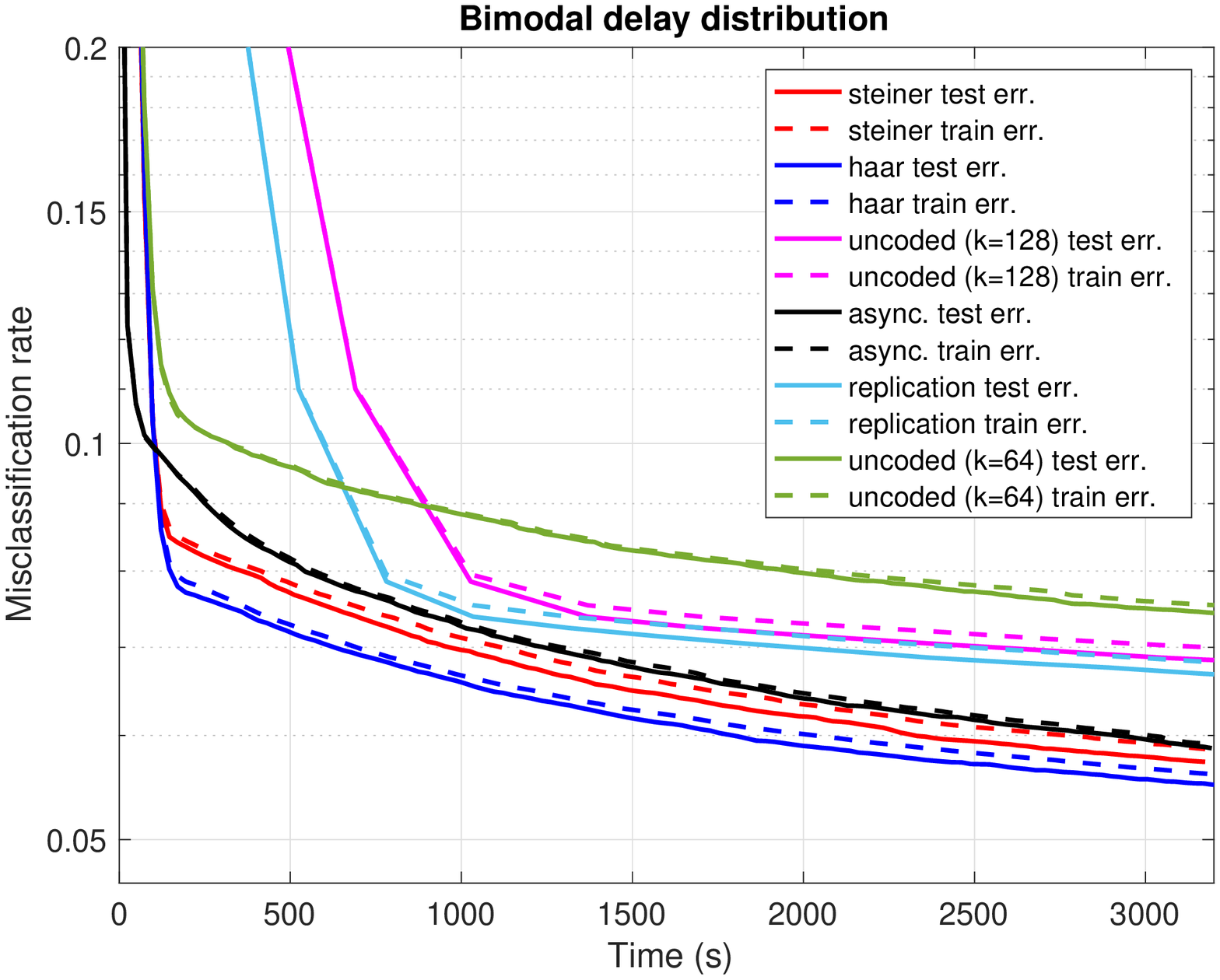}
  \caption{Test and train errors over time (in seconds) for each scheme, for the bimodal delay distribution. Steiner and Haar encoding is done with $k=64$, $\beta=2$.}
  \label{fig:bimodal}
\end{minipage}\hfill
\begin{minipage}{.48\textwidth}
\vfill
  \centering
    \includegraphics[width=2.85in]{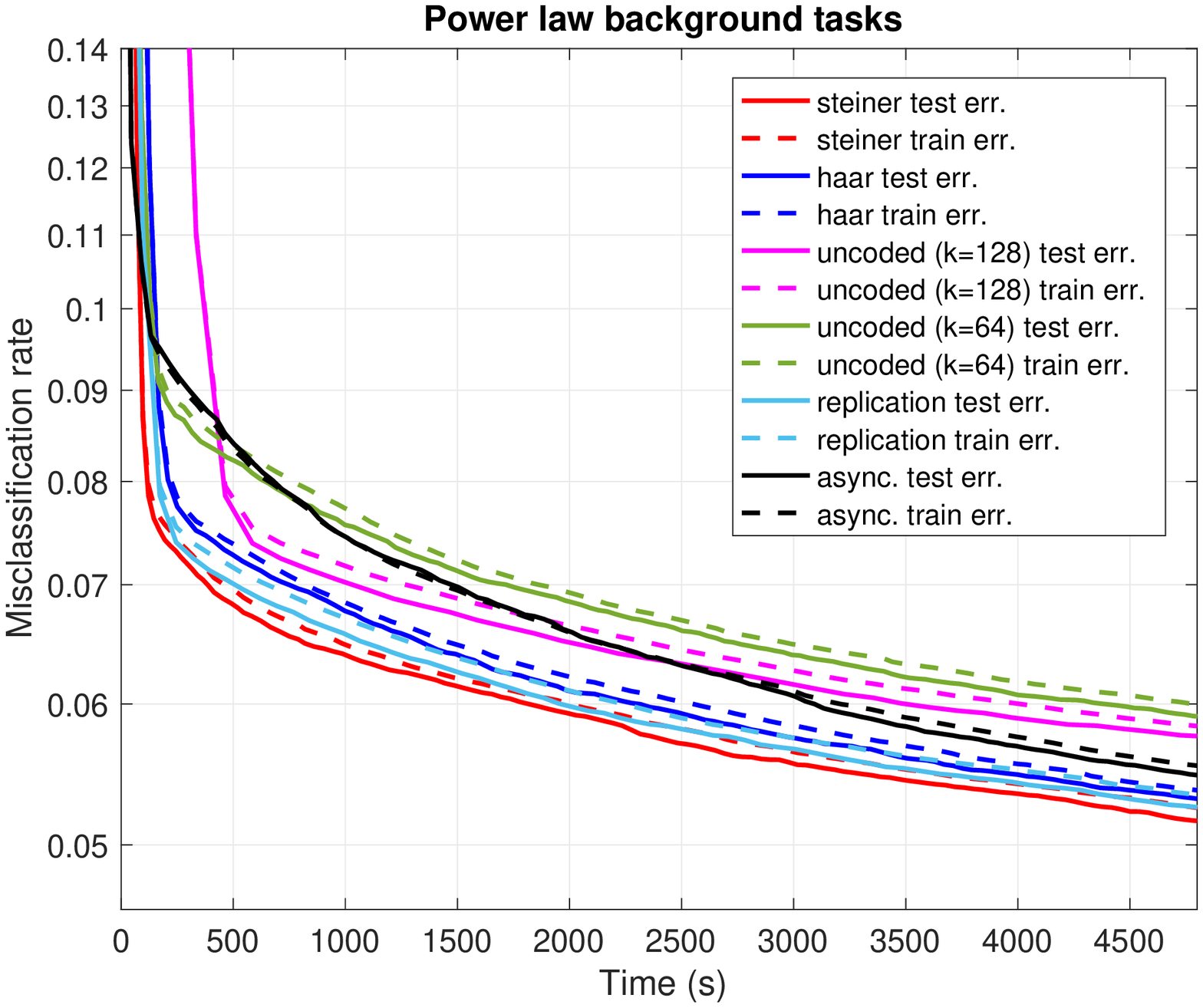}
    \caption{Test and train errors over time (in seconds) for each scheme. Number of background tasks follow a power law. Steiner and Haar encoding is done with $k=80$, $\beta=2$.}
  \label{fig:powerlaw}
\end{minipage}
\end{figure}

\begin{figure}
\centering
\begin{minipage}{.46\textwidth}
  \centering
    \includegraphics[width=2.6in]{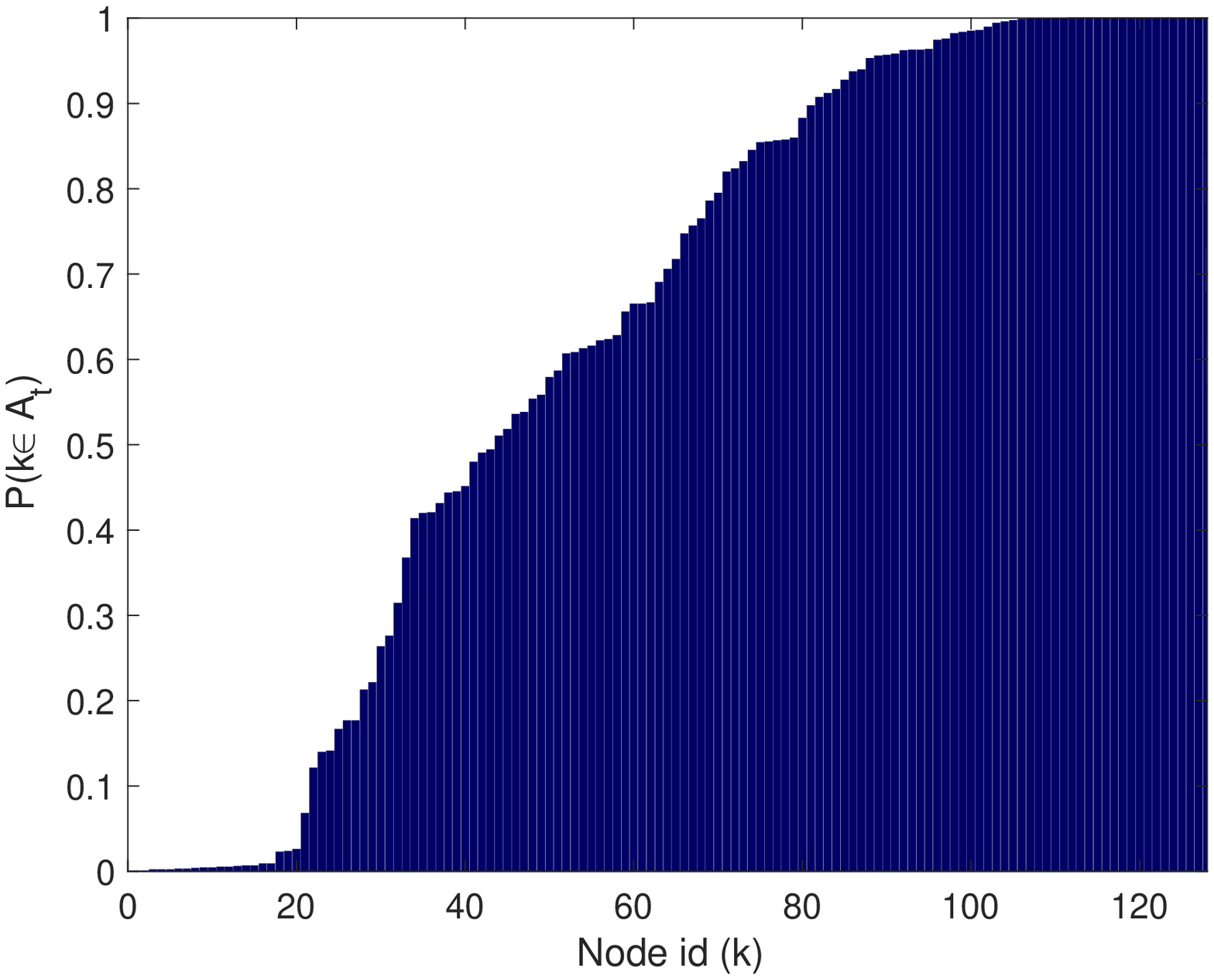}
  \caption{The fraction of iterations each worker node participates in (the empirical probability of the event $\{k \in A_t\}$), plotted for Steiner encoding with $k=80$, $m=128$. The number of background tasks are distributed by a power law with $\alpha = 1.5$ (capped at 50).}
  \label{fig:stat_steiner}
\end{minipage}\hfill
\begin{minipage}{.46\textwidth}
\vfill
  \centering
    \includegraphics[width=2.6in]{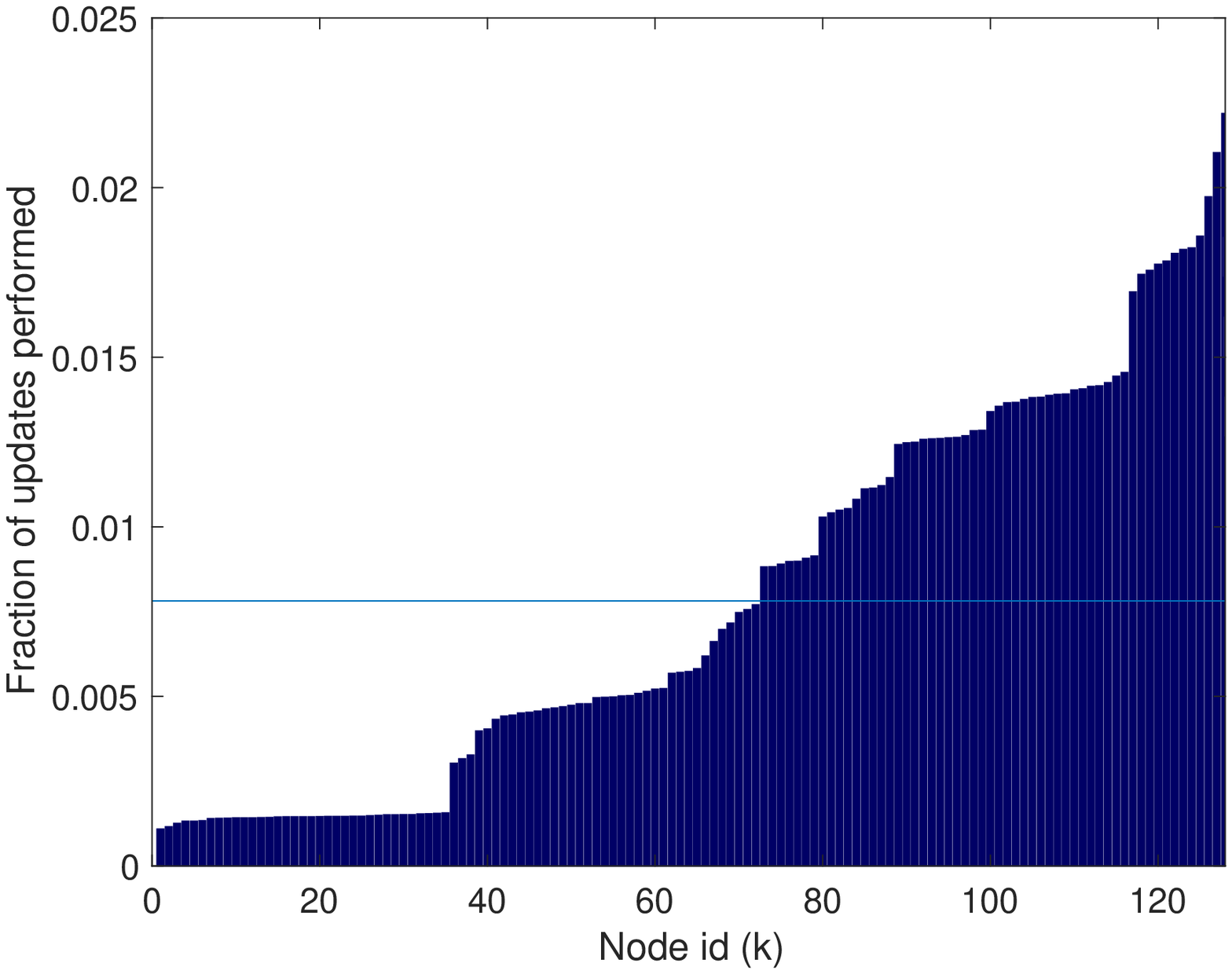}
    \caption{The fraction of updates performed by each node, for asynchronous block coordinate descent. The horizontal line represents the uniformly distributed case. The number of background tasks are distributed by a power law with $\alpha = 1.5$ (capped at 50).}
  \label{fig:stat_async}
\end{minipage}
\end{figure}

Figures~\ref{fig:bimodal} and \ref{fig:powerlaw} shows the evolution of training and test errors as a function of wall clock time. We observe that for each straggler model, either Steiner or Haar encoded optimization dominates all schemes.
Figures~\ref{fig:stat_steiner} and \ref{fig:stat_async} show the statistics of how frequent each node participates in an update, for the case with background tasks, for encoded and asynchronous cases, respectively. We observe that the stark difference in the relative speeds of different machines result in vastly different update frequencies for the asynchronous case, which results in updates with large delays, and a corresponding performance loss.

\subsection{LASSO}
We solve the LASSO problem, with the objective
\begin{align*}
\min_w \frac{1}{2n} \| Xw - y\|^2 + \lambda \| w\|_1^2,
\end{align*}
where $X \in \mathbb{R}^{130,000 \times 100,000}$ is a matrix with i.i.d. $N(0,1)$ entries, and $y$ is generated from $X$ and a parameter vector $w^*$ through a linear model with Gaussian noise:
\begin{align*}
y = Xw^* + \sigma z,
\end{align*}
where $\sigma = 40$, $z \sim N(0,1)$. The parameter vector $w^*$ has 7695 non-zero entries out of 100,000, where the non-zero entries are generated i.i.d. from $N(0,4)$. We choose $\lambda = 0.6$ and consider the sparsity recovery performance of the corresponding LASSO problem, solved using proximal gradient (iterative shrinkage/thresholding algorithm).

We implement the algorithm over 128 \texttt{t2.medium} worker nodes which collectively store the matrix $X$, and a \texttt{c3.4xlarge} master node. We measure the sparsity recovery performance of the solution using the F1 score, defined as the harmonic mean
\begin{align*}
F1 = \frac{2PR}{P+R},
\end{align*}
where $P$ and $R$ are precision recall of the solution vector $\hat w$ respectively, defined as
\begin{align*}
P = \frac{\left| \lbp i: w^*_i \neq 0, \hat w_i \neq 0\rbp\right|}{\left| i: \hat w_i \neq 0 \right|}, \;\; R = \frac{\left| \lbp i: w^*_i \neq 0, \hat w_i \neq 0\rbp\right|}{\left| i: w^*_i \neq 0 \right|}
\end{align*}.

\begin{figure}
\centering
\includegraphics[scale=0.4]{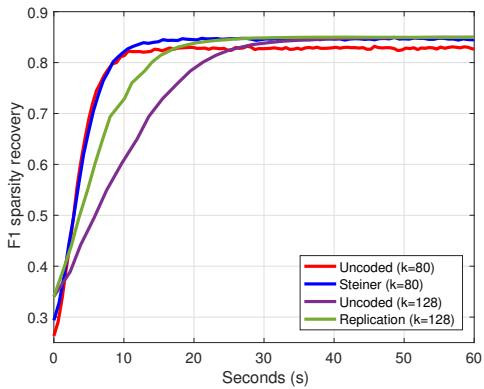}
\caption{Evolution of F1 sparsity recovery performance for each scheme.}
\label{fig:lasso}
\end{figure}
Figure~\ref{fig:lasso} shows the sample evolution of the F1 score of the model under uncoded, replication, and Steiner encoded scenarios, with artificial multi-modal communication delay distribution $q_1 \mathcal{N}(\mu_1, \sigma_1^2) + q_2\mathcal{N}(\mu_2, \sigma_2^2) + q_3\mathcal{N}(\mu_3, \sigma_3^2)$, where $q_1 = 0.8$, $q_2=0.1$, $q_3 = 0.1$; $\mu_1 = 0.2$s, $\mu_2 = 0.6$s, $\mu_3 = 1$s; and $\sigma_1= 0.1$s, $\sigma = 0.2$s, $\sigma_3 = 0.4$s, independently at each node.
We observe that the uncoded case $k=80$ results in a performance loss in sparsity recovery due to data dropped from delayed noes, and uncoded and replication with $k=128$ converges slow due to stragglers, while Steiner coding with $k=80$ is not delayed by stragglers, while maintaining almost the same sparsity recovery performance as the solution of the uncoded $k=128$ case.

\acks{The work of Can Karakus and Suhas Diggavi was supported in part by NSF
grants \#1314937 and \#1514531. The work of Wotao Yin was supported by ONR Grant N000141712162, and NSF Grant DMS-1720237.}

%\newpage

% organize all proofs

\appendix

\section{Proofs of Theorems~\ref{th:gd} and \ref{th:lbfgs}}\label{ap:gdlbfgs}
In the proofs, we will ignore the normalization constants on the objective functions for brevity. We will assume the normalization $\frac{1}{\sqrt{\eta}}$ is absorbed into the encoding matrix $S_A$. Let $\wtild f^A_t := \| S_{A_t} \lp X w_t - y \rp \|^2 + \lambda h(w)$, and $\wtild f^A(w) := \| S_{A_t} \lp X w - y \rp\|^2 + \lambda h(w)$, where we set $A \equiv A_t$. Let $\wtild w_t^*$ denote the solution to the effective ``instantaneous" problem at iteration $t$, \emph{i.e.}, $\wtild w_t^* = \argmin_w \wtild f^A (w)$.

Throughout this appendix, we will also denote
\begin{align*}
w^*&= \argmin_{w} \| Xw - y\|^2 + \lambda h(w) \\
\hat w &= \argmin_{w} \| S_A\lp Xw - y \rp\|^2 + \lambda h(w)
\end{align*}
unless otherwise noted, where $A$ is a fixed subset of $[m]$.

\subsection{Lemmas}

\begin{lemma}\label{lem:solution_ball}
 If $S$ satisfies \eqref{eq:rip} for any $A \subseteq [m]$ with $\left| A \right| \geq k$, for any convex set $C$,
\begin{align*}
    \| X\hat w - y\|^2 \leq \kappa^2 \| X w^* - y\|^2,
\end{align*}
where $\kappa = \frac{1+\epsilon}{1-\epsilon}$, $\hat w = \argmin_{w \in C} \| S_A\lp Xw - y\rp\|^2$, and $w^* = \argmin_{w \in C} \| Xw - y\|^2$.
\end{lemma}
\begin{proof}
Define $e = \hat w - w^*$ and note that
\begin{align*}
\| X\hat w-y\| = \| X w^*-y + Xe\| \leq \| X w^*-y\| + \| Xe\| 
\end{align*}
by triangle inequality, which implies
\begin{align}
\| X\hat w-y\|^2 \leq \lp 1 + \frac{\| Xe\|}{\| X w^*-y\|}\rp^2 \| X w^*-y\|^2 = \lp 1 + \frac{\| Xe\|}{\| X w^*-y\|}\rp^2 \| Xw^* - y\|^2 \label{eq:sol_ball_main}.
\end{align}
Now, for any $c>0$, consider
\begin{align*}
\| Xe \|^2 &\leq \frac{\| S_A Xe\|^2}{1-\epsilon} \overset{\aaaa}{\leq} -2\frac{e^\top X^\top  S_A^\top S_A (Xw^*-y)}{1-\epsilon} \\
&= -2\frac{e^\top X^\top \lp S_A^\top  S_A -cI\rp (Xw^*-y)}{1-\epsilon} - \frac{2c}{1-\epsilon} e^\top X^\top  (Xw^*-y) \\
&\overset{\bbbb}{\leq} -2\frac{e^\top X^\top \lp S_A^\top S_A -cI\rp (Xw^*-y)}{1-\epsilon} \\
&\overset{\cccc}{\leq} 2\frac{ \left\| e^\top X^\top \lp cI -  S_A^\top S_A \rp \right\|}{1-\epsilon} \|Xw^*-y \| \\
&\overset{\dddd}{\leq} 2\frac{ \left\| cI - S_A^\top S_A \right\|}{1-\epsilon} \|Xw^*-y \| \|Xe \|,
\end{align*}
where (a) follows by expanding and re-arranging $\left\| S_A \lp X \hat w -  y \rp\right\|^2 \leq \left\| S_A \lp Xw^* - y\rp \right\|^2$, which is true since $\hat w$ is the minimizer of this function; (b) follows by the fact that since $\hat w \in C$, $e$ represents a feasible direction of the constrained optimization, and thus the convex optimality condition implies $\langle \nabla f(w^*), \hat w - w^* \rangle = e^\top X^\top (Xw^*-y) \geq 0$; (c) follows by Cauchy-Schwarz inequality; and (d) follows by the definition of matrix norm.

Since this is true for any $c>0$, we make the minimizing choice $c=\frac{\lambda_{\max} + \lambda_{\min}}{2}$ (where $\lambda_{\max}$ and $\lambda_{\min}$ represent the largest and smallest eigenvalues of $S_A^\top S_A$, respectively), which gives
\begin{align*}
\frac{\| Xe \|}{\| X\hat w-y\|} \leq \frac{\lambda_{\max} - \lambda_{\min}}{\lambda_{\min}} \leq \frac{2\epsilon}{1-\epsilon}.
\end{align*}
Plugging this back in \eqref{eq:sol_ball_main}, we get the desired result.
\end{proof}

% regularized solution_ball
\begin{lemma}\label{lem:reg_solution_ball}
 If $S$ satisfies \eqref{eq:rip} for any $A \subseteq [m]$ with $\left| A \right| \geq k$,
\begin{align*}
    f(\hat w) \leq \kappa^2 f(w^*),
\end{align*}
where $\kappa = \frac{1+\epsilon}{1-\epsilon}$, $\hat w = \argmin_{w} \| S_A\lp Xw - y\rp\|^2 + \lambda h(w)$, and $w^* = \argmin_w \| Xw - y\|^2 + \lambda h(w)$. 
\end{lemma}
\begin{proof}
Consider a fixed $A_t = A$, and a corresponding 
\begin{align*}
\hat w = \wtild w_t^* \in \argmin_w \| S_A\lp Xw-y\rp \|^2 + \lambda h(w)
\end{align*}
Define
\begin{align*}
\hat w(r) &= \argmin_{w: \lambda h(w) \leq r} \| S_A\lp Xw-y\rp \|^2 \\
 w^*(r) &= \argmin_{w: \lambda h(w) \leq r} \| Xw-y \|^2.
\end{align*}
Finally, define
\begin{align*}
r^* = \argmin_r \| Xw^*(r) - y\|^2 + r.
\end{align*}
Now, consider
\begin{align*}
f(\hat w) &= \| X\hat w - y\|^2 + \lambda h(w) = \min_r \lp \| X\hat w(r) - y\|^2  + r\rp \\
&\leq \| X\hat w(r^*) - y\|^2  + r^* \overset{\aaaa}{\leq} \kappa^2 \| Xw^*(r^*) - y\|^2  + r^* \\
&\leq \kappa^2 \lp  \| Xw^*(r^*) - y\|^2  + r^*\rp = \kappa^2 f(w^*),
\end{align*}
which shows the desired result, where (a) follows by Lemma~\ref{lem:solution_ball}, and by the fact that the set $\lbp w: \lambda h(w) \leq r\rbp$ is a convex set. 
\end{proof}

% constrained solution ball

\begin{lemma}\label{lem:final_arg}
    If
    \begin{align*}
        \wtild f^A_{t+1} - \wtild f^A\lp \wtild w^*_t\rp \leq \gamma \lp \wtild f^A_{t} - \wtild f^A\lp \wtild w^*_t\rp\rp
    \end{align*}
    for all $t>0$, and for some $0<\gamma<1$, where $\wtild w^*_t \in \argmin_w \wtild f_t^A$, then
    \begin{align*}
        f(w_t) \leq \lp \kappa\gamma\rp^t f(w_0) + \frac{\kappa^2\lp \kappa-\gamma\rp}{1-\kappa\gamma}f\lp w^*\rp,
    \end{align*}
    where $\kappa = \frac{1+\epsilon}{1-\epsilon}$.
\end{lemma}
\begin{proof}
    Since for any $w$,
\begin{align*}
\lp 1 - \epsilon\rp \| Xw-y\|^2 \leq \lp Xw-y\rp^\top  S_A^\top \wtild S_A \lp Xw-y\rp,
\end{align*}
we have
\begin{align*}
\lp 1 - \epsilon\rp f(w) \leq \wtild f^A(w).
\end{align*}
Similarly $\wtild f^A(w) \leq \lp 1 + \epsilon\rp f(w)$, and therefore, using the assumption of the theorem
\begin{align*}
\lp 1 - \epsilon\rp f(w_{t+1}) - \lp 1 + \epsilon\rp f\lp \wtild w_t^*\rp \leq \gamma \lp \lp 1 + \epsilon\rp f(w_t) - \lp 1 - \epsilon\rp f\lp \wtild w_t^*\rp\rp,
\end{align*}
which can be re-arranged into the linear recursive inequality
\begin{align*}
f(w_{t+1}) \leq \kappa\gamma f_t + (\kappa-\gamma) f\lp \wtild w_t^*\rp \overset{\aaaa}{\leq} \kappa\gamma f(w_t) + \kappa^2(\kappa-\gamma) f\lp w^*\rp,
\end{align*}
where $\kappa=\frac{1+\epsilon}{1-\epsilon}$ and (a) follows by Lemma~\ref{lem:reg_solution_ball}. By considering such inequalities for $0\leq \tau \leq t$, multiplying each by $\lp \kappa\gamma\rp^{t-\tau}$ and summing, we get
\begin{align*}
f(w_t) &\leq \lp \kappa \gamma\rp^t f(w_0) + \kappa^2(\kappa-\gamma)f\lp  w^*\rp\sum_{\tau=0}^{t-1} \lp \kappa\gamma\rp^{\tau}\\
&\leq \lp \kappa \gamma\rp^t f(w_0) + \frac{\kappa^2\lp \kappa-\gamma\rp}{1-\kappa\gamma}f\lp w^*\rp.
\end{align*}
\end{proof}

\begin{lemma}\label{lem:strong_conv}
Under the assumptions of Theorem~\ref{th:lbfgs}, $\wtild f^A(w)$ is $\lp 1 - \epsilon\rp(\mu+\lambda)$-strongly convex.
\end{lemma}
\begin{proof}
It is sufficient to show that the minimum eigenvalue of $\wtild X_A^\top \wtild X_A$ is bounded away from zero. This can easily be shown by the fact that
\begin{align*}
u^\top\wtild X_A^\top \wtild X_Au = u^\top X^\top S_A^\top S_A X u \geq \lp 1 - \epsilon\rp \| Xu\|^2 \geq \lp 1 - \epsilon\rp\mu \|u\|^2,
\end{align*}
for any unit vector $u$.
\end{proof}

\begin{lemma}\label{lem:rotation_bound}
Let $M \in \mathbb{R}^{p \times p}$ be a symmetric positive definite matrix, with the condition number (ratio of maximum eigenvalue to the minimum eigenvalue) given by $\kappa$. Then, for any unit vector $u$,
\begin{align*}
\frac{u^\top Mu}{\| Mu\|} \geq \frac{2\sqrt{\kappa}}{\kappa+1}.
\end{align*}
\end{lemma}
\begin{proof}
We point out that this is a special case of Kantorovich inequality, but provide a dedicated proof here for completeness.

Let $M$ have the eigen-decomposition $M=Q^\top DQ$, where $Q$ has orthonormal columns, and $D$ is a diagonal matrix with positive, decreasing entries $d_1 \geq d_2 \geq \dots \geq d_n$, with $\frac{d_1}{d_n}=\kappa$. Let $y = \lp Qu \rp^{\circ 2}$, where $^{\circ 2}$ denotes entry-wise square. Then the quantity we are interested in can be represented as
\begin{align*}
\frac{\sum_{i=1}^n d_i y_i}{\sqrt{ \sum_{i=1}^n d_i^2 y_i } },
\end{align*}
which we would like to minimize subject to a simplex constraint $\mathbf{1}^\top y = 1$. Using Lagrange multipliers, it can be seen that the minimum is attained where $y_1 = \frac{1}{1+\kappa}$, $y_n = \frac{\kappa}{1+\kappa}$, and $y_i = 0$ for $i \neq 1,n$. Plugging this back the objective, we get the desired result
\begin{align*}
\frac{u^\top Mu}{\| Mu\|} \geq \frac{2\sqrt{\kappa}}{\kappa+1}.
\end{align*}
\end{proof}

\begin{proof}
%[Proof of Lemma~\ref{lem:hessian_stability}]
[Proof of Lemma~1]
Define $\breve S_t := S_{A_t \cap A_{t-1}}$. First note that
\begin{align}
r_t^\top u_t &= \lp X^\top \breve S_t^\top \breve S_t \lb (Xw_t - y) - (Xw_{t-1} - y)\rb\rp^\top \lp w_t - w_{t-1}\rp \notag\\
& = \lp w_t - w_{t-1}\rp^\top X^\top \breve S_t^\top \breve S_t X\lp w - w_{t-1}\rp \notag\\
&\geq \delta \mu \|u_t\|^2, \label{eq:bd1}
\end{align}
by 
%\eqref{eq:overlap}. 
(5)
Also consider
\begin{align*}
\frac{\|r_t \|^2}{r_t^\top u_t} = \frac{\lp w_t - w_{t-1}\rp^\top \lp X^\top \breve S_t^\top \breve S_t X \rp^2\lp w_t - w_{t-1}\rp }{\lp w_t - w_{t-1}\rp^\top X^\top \breve S_t^\top \breve S_t X\lp w_t - w_{t-1}\rp},
\end{align*}
which implies
\begin{align*}
\epsilon \mu \leq \frac{\|r_t \|^2}{r_t^\top u_t}  \leq (1+\epsilon) M,
\end{align*}
again by 
(4).
%\eqref{eq:rip}.
Now, setting $j_\ell = t - \wtild \sigma + \ell$, consider the trace
\begin{align*}
\trace{B_t^{(\ell+1)}} &= \trace{ B_t^{(\ell)}} - \trace{\frac{B_t^{(\ell)}u_{j_\ell} u_{j_\ell}^\top B_t^{(\ell)} }{u_{j_\ell}^\top B_t^{(\ell)} u_{j_\ell}}} + \trace{\frac{r_{j_\ell} r_{j_\ell}^\top}{r_{j_\ell}^\top u_{j_\ell}} } \\
&\leq \trace{ B_t^{(\ell)}} +  \trace{\frac{r_{j_\ell} r_{j_\ell}^\top}{r_{j_\ell}^\top u_{j_\ell}} } \\
&=  \trace{ B_t^{(\ell)}} +\frac{\|r_{j_\ell} \|^2}{r_{j_\ell}^\top u_{j_\ell}} \\
&\leq  \trace{ B_t^{(\ell)}} +(1+\epsilon) M,
\end{align*}
which implies $\trace{B_t} \leq (1+\epsilon) M \lp \wtild \sigma + d\rp$. It can also be shown (similar to \cite{BerahasNocedal_16}) that
\begin{align*}
\det\lp B_t^{(\ell+1)}  \rp &=  \det\lp B_t^{(\ell)}\rp \cdot\frac{r^\top_{j_\ell} u_{j_\ell}}{u_{j_\ell}^\top B_t^{(\ell)} u_{j_\ell}} \\
&= \det\lp B_t^{(\ell)}\rp \cdot \frac{r^\top_{j_\ell} u_{j_\ell}}{\| u_{j_\ell}\|^2} \cdot \frac{\| u_{j_\ell}\|^2}{u_{j_\ell}^\top B_t^{(\ell)} u_{j_\ell}} \\
&\geq  \det\lp B_t^{(\ell)}\rp \frac{\delta \mu}{(1+\epsilon) M \lp \wtild \sigma + d\rp},
\end{align*}
which implies $\det\lp B_t \rp \geq \det\lp B_t^{(0)}\rp\lp\frac{\delta \mu}{(1+\epsilon) M \lp \wtild \sigma + d\rp}\rp^{\wtild \sigma}$. Since $B_t \geq 0$, its trace is bounded above, and its determinant is bounded away from zero, there must exist $0 < c_1 \leq c_2$ such that
\begin{align*}
c_1 I \preceq B_t \preceq c_2 I.
\end{align*}
\end{proof}

\subsection{Proof of Theorem~\ref{th:gd}}
The proof of the first part of the theorem is a special case of the proof of Theorem~\ref{th:prox} (with $\lambda=0$, and the smooth regularizer incorporated into $p(w)$) and thus we omit this proof and refer the reader to Appendix~\ref{ap:prox}. We prove the second part here.

Note that because of the condition in \eqref{eq:rip}, we have
\begin{align*}
(1-\epsilon) \preceq S_A^\top S_A \preceq (1+\epsilon) I, \\
(1-\epsilon) \preceq S_D^\top S_D \preceq (1+\epsilon) I.
\end{align*}
Using smoothness of the objective, and the choices $d_t =- \nabla \wtild f^A(w_t) (w_t)$ and $\alpha_t = \alpha$, we have
\begin{align*}
&\wtild f^A\lp w_{t+1}\rp - \wtild f^A(w_t) \leq \alpha \nabla \wtild f^A(w_t)(w_t)^\top d_t + \frac{1}{2} \alpha^2 d_t^\top X^\top S_A^\top S_A X d_t + \frac{L}{2} \alpha^2\| d_t\|^2 \\
&\leq -\alpha \lp 1 - \frac{(1+\epsilon)M + L}{2} \alpha \rp \left\|\nabla \wtild f^A(w_t)\right \|^2 = -\frac{2\zeta \lp 1 - \zeta\rp}{(1+\epsilon)M + L}  \|\nabla \wtild f^A(w_t) \|^2 \\
&\overset{\aaaa}{\leq} -\frac{4\nu \zeta \lp 1 - \zeta\rp}{M\lp 1 + \epsilon\rp + L} \lp \wtild f^A\lp w_t\rp - \wtild f^A\lp \wtild w_t^*\rp\rp,
\end{align*}
where (a) follows by strong convexity.
Re-arranging this inequality, and using the definition of $\gamma$, we get
\begin{align*}
    \wtild f^A_{t+1} - \wtild f^A\lp \wtild w_t^*\rp \leq \gamma \lp \wtild f^A_t - \wtild f^A\lp \wtild w_t^*\rp\rp,
\end{align*}
which, using Lemma~\ref{lem:final_arg}, implies the result.

\subsection{Proof of Theorem~\ref{th:lbfgs}}
Since $h(w)$ is constrained to be quadratic, we can absorb this term into the error term to get
\begin{align*}
\min_w \left\| \lb \begin{array}{cc} S & 0 \\ 0 & I\end{array}\rb \lp \lb \begin{array}{c} X \\ \sqrt{\lambda} I\end{array}\rb w - \lb\begin{array}{c} y \\ 0 \end{array} \rb\rp\right\|.
\end{align*}
Note that as long as $S$ satisfies \eqref{eq:rip}, the effective encoding matrix $\diag\lp \lb S, I\rb\rp$ also satisfies the same. Therefore, without loss of generality we can ignore $h(w)$, and assume
\begin{align*}
(\mu + \lambda) I \preceq X^\top X \preceq (M+\lambda) I.
\end{align*}
% here.... finish it then go to lemmas.
We also define $\lambda_{\min}=1-\epsilon$ and $\lambda_{\max}=1+\epsilon$ for convenience. Using convexity and the closed-form expression for the step size, we have
\begin{align*}
&\wtild f^A\lp w_{t+1}\rp - \wtild f^A(w_t) \leq \alpha_t \nabla \wtild f^A(w_t)^\top d_t + \frac{1}{2} \alpha_t^2 d_t^\top X^\top  S_A^\top  S_A X d_t \\
&= - \frac{ \rho\lp \nabla \wtild f^A(w_t)^\top d_t \rp^2}{d_t^\top X^\top  S_D^\top  S_D X d_t} + \frac{1}{2} \frac{\rho^2\lp \nabla \wtild f^A(w_t)^\top d_t\rp^2}{d_t^\top X^\top S_D^\top  S_D X d_t}\cdot \frac{d_t^\top X^\top S_A^\top S_A X d_t}{d_t^\top X^\top S_D^\top S_D X d_t} \\
&= \lp \frac{d_t^\top X^\top \lp \rho^2 S_A^\top S_A - 2\rho S_D^\top S_D\rp  X d_t}{2\lp d_t^\top X^\top S_D^\top S_D X d_t\rp^2}\rp\lp d_t^\top \nabla \wtild f^A(w_t)\rp^2 \\
& \overset{\aaaa}{=} -\rho\lp \frac{z^\top \lp  S_D^\top S_D - \frac{\rho}{2} S_A^\top S_A\rp z}{\lp z^\top S_D^\top S_D z\rp^2}\rp \frac{\lp d_t^\top \nabla \wtild f^A(w_t)\rp^2}{\| Xd_t\|^2} \\
&\overset{\bbbb}{\leq} -\rho\lp \frac{\lambda_{\min}-\frac{\rho}{2}\lambda_{\max}}{\lambda_{\min}^2}\rp \frac{\lp d_t^\top \nabla \wtild f^A(w_t)\rp^2}{\| Xd_t\|^2} \overset{\cccc}{\leq} -\frac{\rho}{M+\lambda}\lp \frac{\lambda_{\min}-\frac{\rho}{2}\lambda_{\max}}{\lambda_{\min}^2}\rp \frac{\lp d_t^\top \nabla \wtild f^A(w_t)\rp^2}{\| d_t\|^2} \\ &\overset{\dddd}{=} -\frac{\rho}{M+\lambda}\lp \frac{\lambda_{\min}-\frac{\rho}{2}\lambda_{\max}}{\lambda_{\min}^2}\rp \frac{\lp \nabla \wtild f^A(w_t)^\top B_t \nabla \wtild f^A(w_t)\rp^2}{\| B_t \nabla \wtild f^A(w_t)\|^2} \\
&\overset{\eeee}{\leq} -\frac{4\rho}{M+\lambda}\lp \frac{\lambda_{\min}-\frac{\rho}{2}\lambda_{\max}}{\lambda_{\min}^2}\rp \frac{c_1 c_2}{\lp c_1 + c_2\rp^2}\| \nabla \wtild f^A(w_t)\|^2 \\
&\overset{\ffff}{\leq} -\frac{8(\mu+\lambda)\rho}{M+\lambda}\lp \frac{\lambda_{\min}-\frac{\rho}{2}\lambda_{\max}}{\lambda_{\min}^2}\rp\frac{c_1 c_2}{\lp c_1 + c_2\rp^2} \lp \wtild f\lp w_t\rp - \wtild f\lp \wtild w_t^*\rp\rp\\
&\overset{\gggg}{=} -\frac{4(\mu+\lambda) c_1 c_2}{(M+\lambda)(1+\epsilon)\lp c_1 + c_2\rp^2}\lp \wtild f\lp w_t\rp - \wtild f\lp \wtild w_t^*\rp\rp \overset{\hhhh}{=} -\lp 1-\gamma\rp \lp \wtild f^A\lp w_t\rp - \wtild f^A\lp \wtild w_t^*\rp\rp.
\end{align*}
where (a) follows by defining $z = \frac{Xd_t}{\| Xd_t\|}$; (b) follows by \eqref{eq:rip}; (c) follows by the assumption that $X^\top X \preceq (M+\lambda )I$; (d) follows by the definition of $d_t$; (e) follows by 
Lemmas~\ref{lem:rotation_bound} 
and \ref{lem:hessian_stability};
%\ref{lem:hessian_stability}; 
(f) follows by strong convexity of $\wtild f$ (by Lemma~\ref{lem:strong_conv}), which implies $\| \nabla \wtild f^A (w_t)\|^2 \geq 2(\mu+\lambda )\lp \wtild f\lp \theta_t\rp - \wtild f\lp \wtild w_t^*\rp\rp$; (g) follows by choosing $\rho = \frac{\lambda_{\min}}{\lambda_{\max}}$; and (h) follows using the definition of $\gamma$.

Re-arranging the inequality, we obtain
\begin{align*}
    \wtild f^A_{t+1} - \wtild f^A\lp \wtild w_t^*\rp \leq \gamma \lp \wtild f^A_t - \wtild f^A\lp \wtild w_t^*\rp\rp,
\end{align*}
and hence applying first Lemma~\ref{lem:final_arg}, we get the desired result.

%  with $\bar w = \wtild w_t$, and then Lemma~\ref{lem:solution_ball}, we get the desired result.

\section{Proofs of Theorem~\ref{th:prox}}\label{ap:prox}
Throughout this appendix, we will define $p(w) = \frac{1}{2}\| Xw - y\|^2$ and $\wtild p_t(w) = \frac{1}{2}\| S_{A_t} \lp Xw-y \rp\|^2$ for convenience, where the normalization by $\sqrt{\eta}$ is absorbed into $S_A$. We will omit the normalization by $n$ for brevity. Let us also define 
\begin{align*}
w^* = \argmin_w p(w) + \lambda h(w)
\end{align*}
to be the true solution of the optimization problem.

By $M$-smoothness of $p(w)$,
\begin{align}
p(w_{t+1}) &\leq p(w_{t}) + \langle \nabla p(w_t), w_{t+1}-w_t\rangle + \frac{M}{2} \| w_{t+1} - w_t\|^2 \notag \\
&\leq p(w^*) - \langle \nabla p(w_t), w^* - w_t\rangle + \langle \nabla p(w_t), w_{t+1}-w_t\rangle + \frac{M}{2} \| w_{t+1} - w_t\|^2 \notag\\
&\leq p(w^*) - \langle \nabla p(w_t), w^* - w_t\rangle + \langle \nabla p(w_t), w_{t+1}-w_t\rangle + \frac{1}{2\alpha} \| w_{t+1} - w_t\|^2 \label{eq:one}
\end{align}
where the second line follows by convexity of $p$, and the third line follows since $\alpha < \frac{1}{M}$. Since $w_{t+1} = \argmin_w \wtild F_t(w)$, by optimality conditions
\begin{align}
0 \in \partial h(w_{t+1}) + \nabla \wtild p_t(w_t) + \frac{1}{\alpha} \lp w_{t+1} - w_t\rp. \label{eq:optimality}
\end{align}
Since $h$ is convex, any subgradient $g \in \partial h$ at $w=w_{t+1}$ satisfies
\begin{align*}
h(w^*) \geq h(w_{t+1}) + \langle g, w^* - w_{t+1}\rangle,
\end{align*}
and therefore \eqref{eq:optimality} implies
\begin{align}
h(w^*) \geq h(w_{t+1}) - \langle \nabla \wtild p_t(w_t) , w^* - w_{t+1}\rangle -   \frac{1}{\alpha} \langle w_{t+1} - w_t , w^* - w_{t+1}\rangle \label{eq:two}.
\end{align}
Combining \eqref{eq:one} and \eqref{eq:two},we have
\begin{align}
f(w_{t+1}) &\leq f(w^*) + \langle \nabla p(w_t) - \nabla \wtild p_t(w_t), w_{t+1} - w^*\rangle \notag \\
&\quad - \frac{1}{\alpha} \langle w_t - w_{t+1}  , w^* - w_{t+1}\rangle + \frac{1}{2\alpha} \| w_t - w_{t+1}\|^2  \notag\\
&= f(w^*) + \langle \nabla p(w_t) - \nabla \wtild p_t(w_t), w_{t+1} - w^*\rangle \notag\\
&\quad + \frac{1}{2\alpha} \lp \| w_t \|^2 -2 w_t^\top w^* + \| w^*\|^2 +2 w_{t+1}^\top w^* - \| w^*\|^2 - \| w_{t+1}\|^2 \rp \notag\\
&= f(w^*) + \langle \nabla p(w_t) - \nabla \wtild p_t(w_t), w_{t+1} - w^*\rangle \notag\\
&\quad + \frac{1}{2\alpha} \lp \| w_t - w^* \|^2 - \| w_{t+1} - w^*\|^2 \rp \label{eq:main}
\end{align}
Define $\Delta = I - S_A^\top S_A$, and consider the second term on the right-hand side of \eqref{eq:main}.
\begin{align*}
&\left\langle \nabla p(w_t) - \nabla \wtild p_t(w_t), w_{t+1} - w^*\right\rangle = \left\langle X^\top \Delta (Xw_t - y), w_{t+1} - w^*\right\rangle \\
&\quad = \left\langle \Delta (Xw_t - y),  Xw_{t+1} - y \rangle - \langle \Delta (Xw_t - y),  Xw^* - y \right\rangle \\
&\quad= \frac{1}{2} \lb \lp X\lp w_t + w_{t+1}\rp - 2y\rp^\top \Delta \lp X\lp w_t + w_{t+1}\rp - 2y\rp \right. \\
&\quad\quad - \lp Xw_{t+1} - y\rp^\top\Delta\lp Xw_{t+1} - y\rp + \lp Xw^* - y\rp^\top\Delta\lp Xw^* - y\rp \\
&\quad\quad \left. - \lp X\lp w_t + w^*\rp - 2y\rp^\top \Delta \lp X\lp w_t + w^*\rp - 2y\rp\rb \\
&\quad= 2 \lp X\lp \frac{w_t + w_{t+1}}{2}\rp - y\rp^\top \Delta \lp X\lp \frac{w_t + w_{t+1}}{2}\rp - y\rp \\
&\quad\quad - 2 \lp X\lp \frac{w_t + w^*}{2}\rp - y\rp^\top \Delta \lp X\lp \frac{w_t + w^*}{2}\rp - y\rp \\
&\quad\quad - \frac{1}{2}\lp Xw_{t+1} - y\rp^\top\Delta\lp Xw_{t+1} - y\rp + \frac{1}{2}\lp Xw^* - y\rp^\top\Delta\lp Xw^* - y\rp \\
&\quad\leq 4\epsilon p \lp \frac{w_t + w_{t+1}}{2}\rp + 4\epsilon p \lp \frac{w_t + w^*}{2}\rp + \epsilon p(w_{t+1}) +  \epsilon p(w^*) \\
&\quad\overset{\aaaa}{\leq} \epsilon \lb 4p(w_t) + 3p(w_{t+1}) + 3p(w^*) \rb \\
&\quad\leq  \epsilon \lb 4f(w_t) + 3f(w_{t+1}) + 3f(w^*) \rb,
\end{align*}
where (a) if by convexity of $p(w)$ and Jensen's inequality, and the last line follows by non-negativity of $h$. Plugging this back in \eqref{eq:main},
\begin{align*}
\lp 1 - 3\epsilon\rp f(w_{t+1}) - 4\epsilon f(w_t) &\leq \lp 1 + 3\epsilon\rp f(w^*)  + \frac{1}{2\alpha} \lp \| w_t - w^* \|^2 - \| w_{t+1} - w^*\|^2 \rp.
\end{align*}
Adding this for $t=1,\dots,(T-1)$,
\begin{align*}
\lp 1 - 7\epsilon \rp \sum_{t=1}^T f(w_t) &\leq  (T-1) \lp 1 + 3\epsilon\rp f(w^*) + 4\epsilon f(w_0) + \frac{1}{2\alpha} \lp \| w_{0} - w^* \|^2 - \| w_{T} - w^* \|^2\rp \\
& \leq T  \lp 1 + 3\epsilon\rp f(w^*)  + 4\epsilon f(w_0)+ \frac{1}{2\alpha} \| w_{0} - w^* \|^2.
\end{align*}
Defining $\bar f_t = \frac{1}{T}\sum_{t=1}^T f(w_t)$, and $\kappa = \frac{1+3\epsilon}{1-7\epsilon}$, we get
\begin{align*}
\bar f_T - \kappa f(w^*) \leq \frac{ 4\epsilon f(w_0) + \frac{1}{2\alpha} \| w_{0} - w^* \|^2}{\lp 1-7\epsilon \rp T},
\end{align*}
which proves the first part of the theorem. To establish the second part of the theorem, note that the
convexity of $h$ implies
\begin{align*}
h(w_t) \geq h(w_{t+1}) + \langle g, w_t - w_{t+1}\rangle,
\end{align*}
where $g \in \partial h(w_{t+1})$. By the optimality condition \eqref{eq:optimality}, this implies
\begin{align*}
h(w_t) \geq h(w_{t+1}) - \langle \nabla \wtild p_t(w_t) , w_t - w_{t+1}\rangle + \frac{1}{\alpha} \| w_{t+1} - w_t \|^2.
\end{align*}
Combining this with the smoothness condition of $p(w)$,
\begin{align*}
p(w_{t+1}) \leq p(w_t) + \langle \nabla p(w_t), w_{t+1}- w_t\rangle + \frac{M}{2} \| w_{t+1} - w_{t}\|^2
\end{align*}
and using the fact that $\alpha < \frac{1}{M}$, we have
\begin{align*}
f(w_{t+1}) \leq f(w_t) + \left\langle \nabla p(w_t) - \nabla \wtild p_t(w_t), w_{t+1} - w_t\right\rangle - \frac{1}{2\alpha} \| w_t - w_{t+1}\|^2.
\end{align*}
As in the previous analysis, we can show that
\begin{align*}
 \left\langle \nabla p(w_t) - \nabla \wtild p_t(w_t), w_{t+1} - w_t\right\rangle \leq \epsilon \lb 7f(w_t) + 3f(w_{t+1}) \rb,
\end{align*}
and therefore
\begin{align*}
f(w_{t+1}) &\leq \frac{1+7\epsilon}{1-3\epsilon}  f(w_t) - \frac{1}{2\alpha (1-3\epsilon)} \| w_t - w_{t+1}\|^2 \\
&\leq \frac{1+7\epsilon}{1-3\epsilon}  f(w_t).
\end{align*}

\section{Proof of Theorem~\ref{th:bcd}}\label{ap:bcd}
%Let $v_t^*$ be an optimum that has the the values in $A_t^c$ equal to their current value, \emph{i.e.}, $\min g(v) = g(v_t^*)$, and $\lb v_t^* \rb_{A_t^c} = \lb v_t \rb_{A_t^c}$. We will show in Lemma~\ref{lem:opt_equivalence} that such a point always exists. For a subset of nodes $A \subseteq [m]$, we will use the notation $\nabla_A$ to mean gradient with respect to variables held by the nodes in $A$, \emph{i.e.}, for a function $f$,
%\begin{align*}
%\lb \nabla_A f \rb_i := \lbp
%\begin{array}{ll}
%\nabla_i f, & i \in A \\
%0, & i \in A^c.
%\end{array}
%\right.
%\end{align*}
For an iterate $v_t$, let $w_t := Sv_t$. Define the solution set $\mathcal{S} = \argmin_w g(w)$, and $w^*_t = \mathcal{P}_\mathcal{S} \lp w_t\rp$, where $\mathcal{P}_\mathcal{S}\lp \cdot \rp$ is the projection operator onto the set $\mathcal{S}$. Let $v^*_t$ be such that $w_t^* = S^\top v_t^*$, which always exists since $S$ has full column rank.

We also define $L':=L(1+\epsilon)$, and $g^* = \min_w g(w) = g(w_t^*)$ for any $t$. %Throughout this appendix, for a subset of nodes $A \subseteq [m]$, $S_A$ refers to the submatrix formed by horizontally stacking the blocks $S_i$, for $i \in A$.

\subsection{Lemmas}

\begin{lemma}\label{lem:bcd_smooth}
$\wtild g(v)$ is $L'$-smooth.
\end{lemma}
\begin{proof}
For any $u,v,$
\begin{align*}
\wtild g(u) &= g( S^\top u) \leq g( S^\top v) + \langle \nabla g( S^\top v), S^\top (u-v)\rangle + \frac{L}{2}\| S^\top (u-v)\|^2, \\
&\stackrel{(a)}{\leq} g( S^\top v) + \langle  S \nabla g( S^\top v), u-v\rangle + \frac{L (1+\epsilon)}{2}\| u-v\|^2, \\
&\stackrel{(b)}=  \wtild g(v) + \langle \nabla \wtild g(v), u-v\rangle + \frac{L (1+\epsilon)}{2}\| u-v\|^2,
\end{align*}
where $(a)$ follows from smoothness of $g$, and from $(m, \eta, \epsilon)$-BRIP property,
and $(b)$ is by the chain rule of derivatives and the definition of $\wtild g(v)$.
Therefore $\wtild g$ is $L(1+\epsilon)$-smooth.
\end{proof}

\begin{lemma}\label{lem:opt_equivalence}
For any $t$,
\begin{align*}
\wtild g^* := \min_{v} \wtild g(v) = \min_{w} g(w) =: g^*.
\end{align*}
\end{lemma}
\begin{proof}
It is clear that 
\begin{align*}
\min_v \tilde g(v) = \min_v g(S^\top v) \geq \min_w g(w).
\end{align*}
To show the other direction, set $v^* = S (S^\top S)^{-1} w^*$, where $S^\top S$ is invertible since $S$ has full column rank. Then $g(w^*) = \tilde g(v^*) \geq \min_v \tilde g(v)$. 
%Let $w^* \in \argmin g(w)$, and thus $g(w^*) = \min_w g(w)$. Since $\wtild g(v) = g(Sv)$, to prove the result it is sufficient to show that there exists $v$ such that $Sv = w^*$. This holds since $S$ has full row rank and $\wtild g(S^\dagger w^*) = g(w^*)$, where $S^\dagger$ represents the Moore-Penrose pseudo-inverse (other choices are possible).
\end{proof}

\begin{lemma}\label{lem:rsc}
If $g$ is $\nu$-restricted-strongly convex, then
\begin{align*}
g(w) - g^* \geq \nu \| w - w^*\|^2,
\end{align*}
where $w^* = \mathcal{P}_{\mathcal{S}}(w)$. 
\end{lemma}
\begin{proof}
We follow the proof technique in \cite{ZhangYin_13}. We have
\begin{align*}
g(w) &= g^* + \int_0^1 \langle \nabla g(w^* + \tau (w - w^*)), w- w^*\rangle d\tau \\
&= g^* + \int_0^1 \frac{1}{\tau}\langle \nabla g(w^* + \tau (w - w^*)), \tau (w- w^*)\rangle d\tau \\
&\geq g^* + \int_0^1 \frac{1}{\tau}\nu \tau^2 \| w - w^*\|^2 d\tau \\
& = g^* + \nu \| w - w^*\|^2,
\end{align*}
which is the desired result, where in the third line we used $\nu$-restricted strong convexity, and the fact that
\begin{align*}
\mathcal{P}_{\mathcal{S}}(w^* + \tau (w - w^*)) = w^*,
\end{align*}
for all $\tau \in [0,1]$, since $w^* = \mathcal{P}_{\mathcal{S}}(w)$ is thr orthogonal projection.
\end{proof}

\subsection{Proof of Theorem~\ref{th:bcd}}
Recall that the step for block $i$ at time $t$, $\Delta_{i,t}$, is defined by
\begin{align*}
\Delta_{i,t} := \lbp 
\begin{array}{cl}
-\alpha\nabla_i  \wtild g(v_{t-1}), & \text{if $i \in A_t$} \\
0, & \text{otherwise.} \\
\end{array}
\right.
\end{align*}
By smoothness and definition of $\Delta_t$,
\begin{align}
 \wtild g (v_{t+1}) -\wtild g (v_{t})& \leq \langle \nabla \wtild g (v_t), \Delta_t\rangle + \frac{L'}{2} \| \Delta_t\|^2\notag \\
&= \sum_{i \in A_t} \lp\langle \nabla_i \wtild g (v_t), \Delta_{i,t}\rangle + \frac{L'}{2} \| \Delta_{i,t}\|^2\rp \notag\\
&= \sum_{i \in A_t} \lp -\frac{1}{\alpha}\langle \Delta_{i,t}, \Delta_{i,t}\rangle + \frac{L'}{2} \| \Delta_{i,t}\|^2\rp\notag \\
&= -\lp \frac{1}{\alpha} - \frac{L'}{2}\rp  \| \Delta_t\|^2. \label{eq:bcd_initial}
\end{align}

%Define $\wtild v_t^* = \mathcal{P}_{\argmin \wtild g_t}\lp \wtild v_t\rp$, where $\mathcal{P}_C(x) := \argmin_{z \in C} \| z-x\|^2$ is a projection operator. 
Now, for any $t$,
\begin{align}
\wtild g(v_t) - \wtild g^* &\leq \left\langle \nabla \wtild g( v_t), v_t^* - v_t\right\rangle = \left\langle S \nabla g(S^\top v_t), v_t^* - v_t\right\rangle \notag\\ 
&\stackrel{(a)}{\leq} \left\|\nabla g(S^\top v_t) \right\| \cdot \left\| S^\top \lp v_t^* - v_t\rp \right\| = \left\|\nabla g(S^\top v_t) \right\| \cdot \left\| w_t^* - w_t \right\|, \label{eq:opt_gap}
\end{align}
where $(a)$ is due to Cauchy-Schwartz inequality. %Since
%\begin{align*}
%\Delta_i^t = -\frac{\gamma}{L'}\nabla_i g( v^t) = -\frac{\gamma}{L'}\nabla_i f( Sv^t) = -\frac{\gamma}{L'} %s_i^\top \nabla f(Sv^t),
%\end{align*}
Using
\begin{align*}
\Delta_t = -\alpha P_t\lb
\begin{array}{c}
 S_{A_t} \nabla g(S^\top v_t) \\
0
\end{array}
\rb,
\end{align*}
where $P_t$ is a block permutation matrix mapping $\lbp 1,\dots,k\rbp$ to the node indices in $A_t$, we have
\begin{align}
\| \Delta_t\|^2 = \alpha^2 \nabla g(S^\top v_t)^\top S^\top_{A_t} P_t^\top P_t S_{A_t}\nabla g(S^\top v_t) \geq (1-\epsilon) \alpha^2  \left\|\nabla g( S^\top v_t)\right\|^2. \label{eq:delta_f}
\end{align}

Because of \eqref{eq:bcd_initial}, we have
\begin{align*}
\wtild g(v_{t+1}) - \wtild g(v_t) = g(w_{t+1}) - g(w_t) \leq 0,
\end{align*}
and hence $w_t$ is contained in the level set defined by the initial iterate for all $t$, \emph{i.e.},
\begin{align*}
w_t \in \lbp w: g(w) \leq g(w_0)\rbp.
\end{align*}
By the diameter assumption on this set, we have $\| w_t - w_t^*\| \leq R$ for all $t$.
%Since $g$ is coercive, \emph{i.e.}, $\| w \| \to \infty \implies g(w) \to \infty$, this means that there must exist a constant $R$ such that $\| w_t \| \leq R$ for all $t$. Further, since $w_t^* = \mathcal{P}_\mathcal{S} \lp w_t\rp$ is a projection, $\| w_t^*\| \leq \| w_t\|$ and thus by triangle inequality
%\begin{align}
%\| w_t - w_t^* \| \leq 2\| w_t\| \leq 2R. \label{eq:param_norm_bound}
%\end{align}
Using this and \eqref{eq:delta_f} in \eqref{eq:opt_gap}, we get
\begin{align*}
\wtild g( v_t) - \wtild g^* \leq \frac{R}{\alpha} \sqrt{\frac{1}{1-\epsilon}}  \| \Delta_t\|.
\end{align*}

Combining this with \eqref{eq:bcd_initial},
\begin{align*}
\wtild g ( v_{t+1}) - \wtild g( v_{t}) \leq - \frac{(1-\epsilon)\alpha}{R}\lp 1 - \frac{\alpha L'}{2}\rp \lp \wtild g ( v_t) - \wtild g^*\rp^2.
\end{align*}

Defining $\pi_t := \wtild g ( v_t) - \wtild g^*$, and $C := \frac{(1-\epsilon)\alpha}{R}\lp 1 - \frac{\alpha L'}{2}\rp$, this implies
\begin{align*}
\pi_{t+1} \leq \pi_t - C \pi_t^2.
\end{align*}
Dividing both sides by $\pi_t \pi_{t+1}$, and noting that $\pi_{t+1} \leq \pi_t$ due to \eqref{eq:bcd_initial},
\begin{align*}
\frac{1}{\pi_{t}} \leq \frac{1}{\pi_{t+1}} - C \frac{\pi_t}{\pi_{t+1}}  \leq \frac{1}{\pi_{t+1}} - C
\end{align*}
Therefore
\begin{align*}
\frac{1}{\pi_t} \geq \frac{1}{\pi_0} + Ct,
\end{align*}
which implies
\begin{align*}
\pi_t \leq \frac{1}{\frac{1}{\pi_0} + Ct}.
\end{align*}
Since $g(w_t) = g(S^\top v_t) = \wtild g(v_t)$ by definition, and $g^* = \wtild g^*$ by Lemma~\ref{lem:opt_equivalence}, $\pi_t = g(w_t) - g^*$, and therefore we have established the first part of the theorem.

To prove the second part, we make the additional assumption that $g$ satisfies $\nu$-restricted-strong convexity, which, through Lemma~\ref{lem:rsc}, implies
%\begin{align*}
%\langle \nabla g(w), w - w^* \rangle \geq \gamma \| w - w^*\|^2,
%\end{align*}
%for any $w$, where $w^* = \argmin g$. Together with convexity property $g(w) - g^* \geq \langle \nabla g(w), w - w^*\rangle$, this implies
$
g(w) - g^* \geq \nu \| w - w^*\|^2,
$
for $w^* =\mathcal{P}_\mathcal{S}(w)$. Plugging in $w = w_t$ then gives the bound
\begin{align*}
\| w_t - w_t^*\|^2 \leq \frac{1}{\nu} \pi_t.
\end{align*}
Using this bound as well as \eqref{eq:delta_f} in \eqref{eq:opt_gap}, we have
\begin{align*}
\pi_t^2 \leq \frac{\| \Delta_t\|^2}{\nu (1-\epsilon) \alpha^2}   \pi_t.
\end{align*}
Using \eqref{eq:bcd_initial}, this gives
\begin{align*}
\pi_t &\leq \frac{1}{\nu (1-\epsilon) \alpha^2} \lp \frac{1}{\alpha} - \frac{L'}{2}\rp^{-1} \lp \pi_t - \pi_{t+1}\rp,
%&= \frac{(1+\epsilon) L'}{(1-\epsilon) \gamma}\frac{2-\alpha}{2\alpha^3}\lp \pi_t - \pi_{t+1}\rp.
\end{align*}
which, defining $\xi = \frac{1}{\nu (1-\epsilon) \alpha} \lp 1 - \frac{L' \alpha}{2}\rp^{-1}$, results in
\begin{align*}
\pi_t \leq \lp 1 - \frac{1}{\xi}\rp^t \pi_0,
\end{align*}
which shows the desired result.

\section{Full results of the Matrix factorization experiment}\label{ap:movielens}
%\begin{figure}
%    \centering
%    \includegraphics[width=4in]{figs/runtime_comparison_24workers_12selected.png}
%    \caption{Runtime with and without sketching overhead and time spent solving problems with $< 500$ rows, for different sketching techniques. Here, $m = 24$ nodes and $k = m/2$.  Gaussian has a slightly larger sketching overhead time, since it generates its sketch matrix for each new problem; in contrast, Paley and Hadamard draw from the sketch matrix bank.}
%    \label{fig:movielens_time}
%\end{figure}
Tables \ref{t-movielens8} and \ref{t-movielens24} give the test and train RMSE for the Movielens 1-M recommendation task, with a random 80/20 train/test split.
\begin{table}[h] \small
\centering
\begin{tabular}{|c|c|c|c|c|c|ccc}
\hline
& uncoded & replication & gaussian & paley & hadamard\\
\hline
&\multicolumn{5}{|c|}{$m=8$, $k=1$}\\
\hline
train RMSE &0.804 & 0.783  & 0.781 & \textbf{0.775} & 0.779 \\
test RMSE &0.898 & 0.889  & 0.877 & \textbf{0.873} & 0.874 \\
runtime &1.60 & 1.76 & 2.24 & 1.82 & 1.82  \\
\hline
&\multicolumn{5}{|c|}{$m=8$, $k=4$}\\
\hline
train RMSE &0.770 & 0.766 &  0.765 & \textbf{0.763} & 0.765 \\
test RMSE &0.872 & 0.872 & \textbf{0.866} & 0.868 & 0.870 \\
runtime &2.96 & 3.13 & 3.64 & 3.34 & 3.18  \\
\hline
&\multicolumn{5}{|c|}{$m=8$, $k=6$}\\
\hline
train RMSE &0.762 & 0.760  & 0.762 & \textbf{0.758} & 0.760 \\
test RMSE &0.866 & 0.871  & 0.864 & \textbf{0.860} & 0.864 \\
runtime &5.11 & 4.59 & 5.70 & 5.50 & 5.33 \\
\hline
\end{tabular}
\caption{Full results for Movielens 1-M, distributed over $m=8$ nodes total. Runtime is in hours. An uncoded scheme running full batch L-BFGS has a train/test RMSE of 0.756 / 0.861, and a runtime of 9.58 hours.}
\label{t-movielens8}
\end{table}

\begin{table}[h]\small
\centering
\begin{tabular}{|c|c|c|c|c|c|ccc}
\hline
& uncoded & replication  & gaussian & paley & hadamard\\
\hline
&\multicolumn{5}{|c|}{$m=24$, $k=3$}\\
\hline
train RMSE &0.805 & 0.791 & 0.783 & \textbf{0.780} & 0.782 \\
test RMSE &0.902 & 0.893 & 0.880 & \textbf{0.879} & 0.882 \\
runtime &2.60 & 3.22  & 3.98 & 3.49 & 3.49 \\
\hline
&\multicolumn{5}{|c|}{$m=24$, $k=12$}\\
\hline
train RMSE &0.770 &  \textbf{0.764} & 0.767 & \textbf{0.764} & 0.765 \\
test RMSE &0.872 & 0.870 & \textbf{0.866} & 0.868 & 0.868 \\
runtime &4.24 &4.38 & 4.92 & 4.50 & 4.61  \\
\hline
\end{tabular}
\caption{Full results for Movielens 1-M, distributed over $m=24$ nodes total. Runtime is in hours. An uncoded scheme running full batch L-BFGS has a train/test RMSE of 0.757 / 0.862, and a runtime of 14.11 hours.}
\label{t-movielens24}
\end{table}

\newpage

\bibliography{Ref}

%\vskip 0.2in
%\bibliographystyle{ieee}

\end{document}